\documentclass[11pt,3p]{elsarticle}
\usepackage{verbatim}
\usepackage{amsmath,amsthm,amsfonts,amssymb}
\usepackage{algorithm}
\usepackage{enumerate}
\usepackage{amssymb}
\usepackage{graphics}
\usepackage{algpseudocode}
\usepackage{cases}
\usepackage{xcolor,sectsty}
\usepackage{sectsty}

\definecolor{astral}{RGB}{46,116,181}
\definecolor{blue1}{RGB}{0, 150, 255}
\subsectionfont{\color{astral}}
\sectionfont{\color{astral}}
\usepackage{mathtools}
\newtheorem{theorem}{Theorem}[section]
\newtheorem{lemma}[theorem]{Lemma}

\newtheorem{definition}[theorem]{Definition}

\newtheorem{thm}{Theorem}[section]

\newtheorem{corollary}[thm]{Corollary}
\newtheorem{proposition}[thm]{Proposition}

\usepackage{tikz}
\usetikzlibrary{shapes.geometric, arrows.meta, positioning, calc, patterns, decorations.pathreplacing}
\definecolor{darkslategray}{rgb}{0.18, 0.31, 0.31}
\definecolor{warmblack}{rgb}{0.0, 0.26, 0.26}
\journal{arXiv.org}

\newcommand{\mb}{\mathbb}
\newcommand{\mc}[1]{\mathcal {#1}}

\usepackage{graphicx}
\usepackage{mathdots}
\usepackage{MnSymbol}
\usepackage{mathrsfs}

\usepackage[english]{babel}

\usepackage{multirow}

\usepackage[hidelinks,colorlinks=true, urlcolor=blue, linkcolor=red, citecolor=green]{hyperref}

\usepackage{hyperref}
\usepackage{comment}

\usepackage{subfigure}

\usepackage{multicol}

\newcommand\x{\times}

\usepackage[customcolors]{hf-tikz}

\tikzset{style green/.style={
    set fill color=green!50!lime!60,
    set border color=white,
  },
  style cyan/.style={
    set fill color=cyan!90!blue!60,
    set border color=white,
  },
  style orange/.style={
    set fill color=black!10!blue!10,
    set border color=white,
  },
  hor/.style={
    above left offset={-0.15,0.31},
    below right offset={0.15,-0.125},
    #1
  },
  ver/.style={
    above left offset={-0.1,0.3},
    below right offset={0.15,-0.15},
    #1
  }
}

\usepackage{tikz,xcolor}
\definecolor{lime}{HTML}{A6CE39}
\definecolor{lightblue}{rgb}{0.0, 0.0, 0.5}
\DeclareRobustCommand{\orcidicon}{%
	\begin{tikzpicture}
	\draw[lime, fill=lime] (0,0)
	circle [radius=0.16]
	node[white] {{\fontfamily{qag}\selectfont \tiny ID}};
	\draw[white, fill=white] (-0.0625,0.095)
	circle [radius=0.007];
	\end{tikzpicture}
	\hspace{-2mm}
}
\foreach \x in {A, ..., Z}{%
	\expandafter\xdef\csname orcid\x\endcsname{\noexpand\href{https://orcid.org/\csname orcidauthor\x\endcsname}{\noexpand\orcidicon}}
}

\begin{document}

\begin{frontmatter}

\title{
Robust Low-Rank Tensor Completion based on M-product with Weighted Correlated Total Variation
and Sparse Regularization 
}
\author{ Biswarup Karmakar$^a$ Ratikanta Behera$^b$, }
\vspace{.3cm}

\address{  
Department of Computational and Data Sciences, \\
Indian Institute of Science, Bangalore, 560012, India.\\
                        \textit{E-mail$^a$}: \texttt{biswarupk@iisc.ac.in}\\
                        \textit{E-mail$^b$}: \texttt{ratikanta@iisc.ac.in}\\
                        
                        }
\begin{abstract}
The robust low-rank tensor completion problem addresses the challenge of recovering corrupted high-dimensional tensor data with missing entries, outliers, and sparse noise commonly found in real-world applications. Existing methodologies have encountered fundamental limitations due to their reliance on uniform regularization schemes, particularly the tensor nuclear norm and $\ell_1$ norm regularization approaches, which indiscriminately apply equal shrinkage to all singular values and sparse components, thereby compromising the preservation of critical tensor structures. The proposed tensor weighted correlated total variation (TWCTV) regularizer addresses these shortcomings through an $M$-product framework that combines a weighted Schatten-$p$ norm on gradient tensors for low-rankness with smoothness enforcement and weighted sparse components for noise suppression. The proposed weighting scheme adaptively reduces the thresholding level to preserve both dominant singular values and sparse components, thus improving the reconstruction of critical structural elements and nuanced details in the recovered signal. Through a systematic algorithmic approach, we introduce an enhanced alternating direction method of multipliers (ADMM) that offers both computational efficiency and theoretical substantiation, with convergence properties comprehensively analyzed within the $M$-product framework.Comprehensive numerical evaluations across image completion, denoising, and background subtraction tasks validate the superior performance of this approach relative to established benchmark methods.
\end{abstract}

\end{frontmatter}

\section{Introduction}

A tensor is a multi-dimensional generalization of matrices that can represent data in more than two dimensions. Tensor-based methods help us to handle and analyze multi-dimensional data more efficiently than traditional matrix methods. These have been successful in various application domains, such as signal processing \cite{10125009,cichocki2015}, machine learning \cite{8884203}, and scientific computing \cite{BraliNT13}, offering efficient representation of large-scale data with powerful problem-solving capabilities. Real-world data often suffer from corruption and missing values during acquisition and transmission processes. The robust low-rank tensor completion (RLRTC) framework addresses this challenge by reconstructing the underlying low-rank tensor structure from such corrupted and incomplete observations. This framework has proven effective in various applications including traffic data prediction~\cite{quan2023}, image restoration~\cite{8698334}, object detection~\cite{SOBRAL201722}, and recommendation systems~\cite{3278607}. Mathematically, the RLRTC model for the $d$-th order tensor $\mc{X}\in\mb{R}^{n_1\times n_2\times n_3\times\cdots\times n_d}$ is expressed as
\begin{equation}\label{eq:rltc}
    \min_{\mathcal{X}, \mathcal{E}} \Psi_1(\mathcal{X}) + \lambda \Psi_2(\mathcal{E}), \quad \text{subject to} \quad P_{\Omega}(\mathcal{M})=P_{\Omega}(\mathcal{X} + \mathcal{E}) ,
\end{equation}
where $\Psi_1(\mathcal{X})$ denotes the low-rank regularization term (a convex or nonconvex surrogate of the tensor rank), $\Psi_2(\mathcal{E})$ accounts for noise and outliers, and $\lambda > 0$ is a regularization parameter that controls the trade-off between the two terms. Here, $\mathcal{M}$ represents the observed tensor, and $P_{\Omega}(\cdot)$ denotes the projection operator onto the observation set $\Omega$, defined as
\begin{equation*}
  (P_{\Omega}(\mathcal{M}))(i_1,i_2,\dots,i_d) = 
\begin{cases}
    \mathcal{M}(i_1,i_2,\dots,i_d), & \text{when } (i_1,i_2,\dots,i_d) \in \Omega, \\
    0, & \text{otherwise}.
\end{cases}  
\end{equation*}

If the index set $\Omega$ includes all elements, meaning that the entire tensor is observed, the RLRTC model simplifies to the tensor robust principal component analysis (TRPCA) problem. However, if no corruption exists, that is, $\mathcal{E} = 0$, the model becomes equivalent to the tensor completion (TC) problem. Consequently, RLRTC can be seen as a generalization of both TC and TRPCA.
Different tensor factorization frameworks introduce various tensor rank definitions and relaxations, making optimization problem~\eqref{eq:rltc} challenging to solve directly. For instance, CP decomposition~\cite{6979248} aims to represent a tensor as a sum of rank-one tensors, but determining the CP rank is NP-hard, making it computationally expensive and challenging to apply in practice. Similarly, Tucker decomposition~\cite{7265046, PanLingALow2024}, which breaks a tensor into a core tensor and factor matrices, requires unfolding the tensor into matrices along different modes. This unfolding process can lead to a substantial loss of the tensor’s inherent structure and local spatial correlations, especially in high-dimensional data. Following previous studies, the $t$-product based tensor algebra framework proposed by Kilmer \& Martin~\cite{kilmer13} has made substantial contributions to tensor completion research~\cite{8606166,kong18}. In 2015, Kernfeld et al.~\cite{Kernfeld15} generalized tensor-tensor multiplication by developing the $M$-product and $c$-product frameworks using invertible linear transformations. Unlike CP and Tucker decompositions, $M$-SVD~\cite{kilmer2021, Guangjing2020} preserves both the intrinsic low-rank structure and spatial relationships without tensor unfolding, making it effective for applications such as image completion and hyperspectral data recovery.

The literature on tensor data recovery has progressed significantly, particularly with the development of advanced regularization techniques. Initially, the $t$-product-based nuclear norm (TNN) \cite{8606166} was introduced to encourage low-rankness by reducing the singular values of a tensor. Although TNN worked better than older matrix-based methods, it struggled to preserve the local structure of the data because it focused too much on global features. To improve this, total variation (TV) regularization was added to keep things smooth and maintain edges, especially in data like hyperspectral images (HSIs) \cite{qiu2021tnntv}. Although TV was good at maintaining local smoothness, it couldn’t fully capture both low-rank and smoothness properties simultaneously. These approaches \cite{ZhiYangTensor2024, he2016lrtv,wang2018lrtdtv,yokota2016spctv} employ the sum of two regularization terms, requiring careful fine-tuning of trade-off parameters to handle real-world scenarios effectively. To solve this, a tensor correlated total variation (TCTV) \cite{Wang202310990} was created by combining the strengths of both TNN and TV. By applying the tensor nuclear norm in the gradient domain, the TCTV can simultaneously enforce both the low-rank structure and smoothness. This makes TCTV a more effective method for tensor completion, as it handles both global structure and local details better than the previous methods. Previous approaches penalized all singular values equally, leading to the over-penalization of significant components in the optimization process. Recent advances in $t$-product-based frameworks have explored weighted nuclear norm regularizers~\cite{yang2025}. For matrix completion, the Schatten-$p$ quasi-norm with a small $p$ can recover low-rank matrices from fewer observed entries than nuclear norm minimization. Motivated by this observation, we propose a nonconvex formulation using exponential weights with stronger decay properties, applying the tensor weighted Schatten-$p$ norm to the gradient tensor within the $M$-product framework to enhance recovery performance. The discrete Fourier transform (DFT) assumes periodic boundary conditions, whereas the discrete cosine transform (DCT) assumes reflective boundary conditions, leading to a smoother and more natural extension at the signal boundaries~\cite{Ng99}. Thus, using DCT as a transform not only provides better recovery accuracy for low-rankness and smoothness by reducing the estimation bias but also achieves faster computation times.

Moreover, the previously mentioned methods focus solely on the variations among singular values while neglecting the differences among the elements of the sparse tensor. Specifically, these methods often utilize the tensor $\ell_1$ norm, applying uniform shrinkage to all elements of the sparse tensor. This approach can overly penalize significant entries, potentially compromising the overall performance and leading to suboptimal outcomes \cite{Candes2008}. Inspired by the superior performance of weighted $\ell_1$ norm minimization over standard $\ell_1$ norm in enhancing signal sparsity~\cite{zhao2012}, we introduce an adaptive weighting scheme for sparse components.

The main contributions of this paper are as follows:
\begin{itemize}
\item A tensor weighted correlated total variation (TWCTV) regularizer is presented that incorporates exponential weighting in the Schatten-$p$ norm (based on $M$-SVD) on the gradient tensor to enhance low-rank structure and smoothness properties. A weighted $\ell_1$ norm is also proposed for the RLRTC problem to effectively handle sparse corruptions while preserving significant components.

\item An efficient alternating direction method of multipliers-based algorithm is established to solve the resulting optimization problem of RLRTC, along with a theoretical analysis of its convergence properties under nonconvex settings.

\item Comprehensive numerical evaluations are conducted across multiple data, including RGB images, hyperspectral data, and video sequences, where our method outperforms in recovery accuracy compared to contemporary approaches.
\end{itemize}
   
The organization of this paper proceeds as follows: The notations and preliminaries within the tensor $M$-product framework are provided in Section 2. Section 3 discusses the weight design for the low-rank regularizer and sparse term. In Section 4, we present the ADMM-based algorithm to solve the RLRTC model. Section 5 presents a detailed convergence analysis of the proposed method. Section 6 presents extensive experimental validation across various applications. Finally, the conclusions about our work are drawn in Section 7.

\section{Preliminaries}\label{SecPreliminaries}
The following notation is adopted throughout this paper: lowercase letters ($x$) denote scalars, uppercase letters (${X}$) indicate matrices, and calligraphic letters ($\mathcal{X}$) denote tensors. For a $d$-th order tensor {$\mathcal{X} \in \mathbb{R}^{n_1 \times n_2\times\cdots \times n_d}$} (or {$\mb{C}^{n_1\times n_2\times\cdots\times n_d})$}, we use $\mathcal{X}(i_1,i_2,\dots,i_d)$ to denote its $(i_1,i_2,\dots,i_d)$-th entry. We denote $\mathcal{X}(:, \dots, :, i_k, :, \dots, :)$ as the $i_k$-th slice along mode-$k$, obtained by fixing the $k$-th index to $i_k$ and varying all other indices. In particular, when $\mathcal{X} \in \mathbb{R}^{n_1 \times \cdots \times n_d}$, the matrix frontal slice $\mathcal{X}^{(j)} := \mathcal{X}(:, :, i_3, \dots, i_d)$ is indexed by 
$j = \sum_{a=4}^{d}(i_a - 1) \prod_{b=3}^{a-1} n_b + i_3,$ which flattens the multi-index $(i_3, \dots, i_d)$ into a single frontal slice index. We also use $\operatorname{bdiag}(\mc{X})\in\mb{R}^{n_1 n_3\cdots n_d \times n_2 n_3\cdots n_d  }$ as block diagonal matrix whose $j$-th block is $\mc{X}^{(j)}$ for $j=1,\dots,n_3\cdots n_d$.

\begin{definition}[Mode-$k$ product {\cite{Kolda09Rev}}]
For a $d$-th order tensor $\mathcal{X} \in \mathbb{R}^{n_1 \times n_2 \times \dots \times n_d}$ (or $\mathbb{C}^{n_1 \times n_2 \times \dots \times n_d}$), and a matrix $U_k \in \mathbb{C}^{n_k \times n_k}$, the mode-$k$ product $\mathcal{X} \times_k U_k$ is defined elementwise as
\[
(\mathcal{X} \times_k U_k)(i_1, \dots, i_{k-1}, i_k, i_{k+1}, \dots, i_d) = \sum_{l=1}^{n_k} \mathcal{X}(i_1, \dots, i_{k-1}, l, i_{k+1}, \dots, i_d) \cdot U_k(i_k, l),
\]
for $k = 1, 2, \dots, d$. This results in a tensor of the same size $n_1 \times n_2 \times \dots \times n_d$.
\end{definition}

Given an invertible matrix $M_k \in \mathbb{C}^{n_k\times n_k}$, the transform $\mathfrak{L}$ and its corresponding inverse $\mathfrak{L}^{-1}$ are defined as following~\cite{kilmer2021}:
\begin{align*}
\mathfrak{L}(\mathcal{X}) &= \hat{\mathcal{X}} := \mathcal{X} \times_3 M_3 \times_4 M_4 \times_5 \cdots \times_d M_d \in \mathbb{C}^{n_1 \times n_2 \times \cdots \times n_d}, \\
\mathfrak{L}^{-1}(\mathcal{X}) &:= \mathcal{X} \times_3 M_3^{-1} \times_4 M_4^{-1} \times_5 \cdots \times_d M_d^{-1} \in \mathbb{C}^{n_1 \times n_2 \times \cdots \times n_d},
\end{align*}
where $M_k$ can be chosen as the unnormalized DFT matrix, the DCT matrix, or any matrix satisfying $M_k M_k^H = M_k^H M_k = \alpha_k I_k$ ($\alpha_k \neq 0$), with $I_k$ denoting the $n_k \times n_k$ identity matrix and $H$ denoting the conjugate transpose.

\begin{definition}[$M$-product \cite{Kernfeld15}]
For two tensors $\mathcal{X} \in \mathbb{R}^{n_1 \times n_2 \times n_3\times\cdots
\times n_d}$ and $\mathcal{Y} \in \mathbb{R}^{n_2 \times l \times n_3\times \cdots\times n_d}$, their $M$-product under invertible transform $\mathfrak{L}$ is defined as:
\begin{equation*}
\mathcal{X} *_M \mathcal{Y} = \mathfrak{L}^{-1}(\hat{\mathcal{X}} \triangle \hat{\mathcal{Y}}), 
\end{equation*}
where $\triangle$ represents the face-wise product such that $\mc{Z}=\mc{X}*_M\mc{Y}\iff \hat{\mathcal{Z}}^{(j)} = \hat{\mathcal{X}}^{(j)}\hat{\mathcal{Y}}^{(j)}\iff\operatorname{bdiag}(\hat{\mc{Z}})=\operatorname{bdiag}(\hat{\mc{X}})\operatorname{bdiag}(\hat{\mc{Y}}) $ for each $j$-th frontal slice matrix. For the special case where $M$ is the unnormalized DFT matrix, this reduces to the standard $t$-product.
\end{definition}

\begin{definition}[$M$-SVD \cite{kilmer2021}]
Let $\mathcal{X} \in \mathbb{R}^{n_1 \times n_2 \times \dots \times n_d}$ be any tensor, then it can be decomposed as:
\begin{equation}\label{eq:msvd}
\mathcal{X} = \mathcal{U} *_M \mathcal{S} *_M \mathcal{V}^T,
\end{equation}
where $\mathcal{U} \in \mathbb{R}^{n_1 \times n_1 \times \dots \times n_d}$ and $\mathcal{V} \in \mathbb{R}^{n_2 \times n_2 \times \dots \times n_d}$ are orthogonal tensors, $\mathcal{S} \in \mathbb{R}^{n_1 \times n_2 \times \dots \times n_d}$ is an f-diagonal tensor (i.e., each frontal slices are diagonal) and $(.)^T$ denotes the transpose.
\end{definition}

\begin{definition}\cite{kilmer2021}
The tubal rank of tensor $\mathcal{X} \in \mathbb{R}^{n_1 \times n_2 \times \cdots \times n_d}$  is defined as:
\begin{equation*}
\text{rank}_{\text{$M$-SVD}}(\mathcal{X}) = \#\{i: \mathcal{S}(i,i,:,\cdots,:) \neq 0\},
\end{equation*}
where $\#$ counts the number of nonzero tubes in the f-diagonal tensor $\mathcal{S}$.
\end{definition}
\begin{definition}[Tensor norms\cite{kilmer11}]
Let $\mc{X} \in \mathbb{R}^{n_1 \times n_2 \times \dots \times n_d}$, then the Frobenius norm, $\ell_1$ norm, and infinity pre-norm are respectively defined as:
\begin{equation*}
    \|\mc{X}\|_F=\sqrt{\sum_{i_1,\dots,i_d} {\mathcal{X}(i_1,\dots,i_d)}^2}, \quad \|\mathcal{X}\|_1 = \sum_{i_1,\dots,i_d} |\mathcal{X}(i_1,\dots,i_d)|, \quad \|\mathcal{X}\|_{\infty} = \max_{i_1,\dots,i_d} |\mathcal{X}(i_1,\dots,i_d)|.
\end{equation*}
\end{definition}

\begin{definition}\label{def:wsn}
(Tensor weighted Schatten $p$-norm\cite{10496551}). Let $\mathcal{X} \in \mathbb{R}^{n_1 \times n_2 \times \dots \times n_d}$ be a tensor with $M$-SVD~\eqref{eq:msvd}, then for $p \in (0,1)$ its tensor weighted Schatten $p$-norm (quasi-norm) is defined as:
\begin{equation*}
  \|\mathcal{X}\|_{{\mc{W},S_p}} := \left(\frac{1}{\sqrt{c}}\sum_{l=1}^{n_3\cdots n_d} \|\hat{\mathcal{X}}^{(l)}\|_{{\mc{W}^{(l)},S_p}}^p\right)^{\frac{1}{p}}= \left(\frac{1}{\sqrt{c}}\sum_{l=1}^{n_3\cdots n_d} \sum_{i=1}^{\min\{n_1,n_2\}}\mc{W}(i,i,l)|\hat{\mc{S}}(i,i,l)|^p\right)^{\frac{1}{p}}, 
\end{equation*}
where $\|\cdot\|_{\mc{W},{S}_p}$ denotes the weighted Schatten $p$-norm with f-diagonal tensor $\mc{W}$ as weight, and $\hat{\mc{S}}(i,i,l)$ represents $i$-th the singular values of $\hat{\mathcal{X}}^{(l)}$. In the above definition, the constant $c = \alpha_1 \cdots \alpha_d$ arises from the choice of the transform matrices. Also, for $\mathcal{W}_d = \operatorname{diag}(\operatorname{bdiag}(\mathcal{W}))$, the norm $\|\mathcal{X}\|_{\mathcal{W}, S_p}$ is defined as
$\|\mathcal{X}\|_{\mathcal{W}, S_p} = \left( \frac{1}{\sqrt{c}} \left\| \operatorname{bdiag}(\hat{\mathcal{X}}) \right\|_{\mathcal{W}_d, S_p}^p \right)^{1/p},$
where $\|\cdot\|_{\mathcal{W}_d, S_p}$ denotes the weighted Schatten-$p$ norm for matrices.

\end{definition}

\begin{definition}
(Tensor weighted group $\ell_p$-norm) Let $\mathcal{X} \in \mathbb{R}^{n_1 \times n_2 \times \dots \times n_d}$ be any tensor, then its weighted group $\ell_p$ quasi-norm is defined as
\begin{equation*}
\|\mathcal{X}\|_{\mathcal{W},\ell_p} = \left(\frac{1}{\sqrt{c}}\sum_{l=1}^{n_3\cdots n_d} \|\hat{\mathcal{X}}^{(l)}\|_{\mathcal{W}^{(l)},\ell_p}^p\right)^{1/p},
\end{equation*}
where the weighted group norm for each frontal slice $l$ is:
\begin{equation*}
\|\hat{\mathcal{X}}^{(l)}\|_{\mathcal{W}^{(l)},\ell_p} = \|[\mathcal{W}(1,1,l)\|\hat{\mc{X}}(:,1,l)\|_2, \ldots, \mathcal{W}(n_2,n_2,l)\|\hat{\mc{X}}(:,n_2,l)\|_2]\|_{l_p}.
\end{equation*}
Here, $\|\cdot\|_{l_p}$ denotes the group $l_p$ norm~\cite{wang2017} of the $l^{th}$ frontal slice.
\end{definition}

\begin{definition}[Gradient tensor \cite{Wang202310990}]\label{def:gt}
Let $\mathcal{X} \in \mathbb{R}^{n_1 \times n_2 \times \dots \times n_d}$ be any tensor, then its gradient tensor $\mathcal{G}_k$ along mode-$k$ ($k=1,2,3$) is defined as:
\begin{equation*}
\mathcal{G}_k := \nabla_k(\mathcal{X}) = \mathcal{X} \times_k D_{n_k}, 
\end{equation*}
where $D_{n_k}$ is a row circulant matrix of vector $(-1, 1, 0, \ldots, 0)$.
\end{definition}
Based on Definition~\ref{def:gt}, the TV norm of a tensor $\mathcal{X}$ is typically expressed using the $\ell_1$-norm as {$\|\mathcal{X}\|_{\text{TV}} = \sum_{k \in \Gamma} \|\mathcal{G}_k\|_1$,} where $\Gamma$ denotes the set of directional modes. {A rank-consistent inequality between $\mathcal{X}$ and its gradient tensors is given by
\[
R - 1 \;\leq\; \operatorname{rank}_{\mathrm{M\text{-}SVD}}(\mathcal{G}_k) \;\leq\; R,
\]
where $R$ denotes the tubal rank of $\mathcal{X}$, reflecting the structural alignment of low-rank representations.} To jointly exploit spatial smoothness and low-rank correlation, the tensor (M-product based) correlated total variation (M-CTV) norm is defined as $\|\mathcal{X}\|_{\text{M-CTV}} = \frac{1}{\gamma} \sum_{k \in \Gamma} \|\mathcal{G}_k\|_{TNN}$, where $\gamma$ is the cardinality of $\Gamma$ and TNN is the convex tensor nuclear norm.

\section{ Proposed Low-rank and Sparse Regularization}
In this section, we will propose low-rank and sparse regularization for the RLRTC problem \eqref{eq:rltc}. In addition, the gradient tensors $\mathcal{G}_k$ preserve the low-rank and smoothness characteristics of the original tensor $\mc{X}$ due to the linearity of the mode-$k$ products. Following this, we present a nonconvex regularizer based on a tensor weighted Schatten $p$-norm applied to gradient tensors, promoting both low-rankness and smoothness simultaneously.

\begin{definition}
For a tensor $\mathcal{X} \in \mathbb{R}^{n_1 \times n_2 \times \dots \times n_d}$, its tensor weighted correlated total variation (TWCTV) norm is defined as:
\begin{equation*}
\|\mathcal{X}\|_{\text{TWCTV}} := \frac{1}{\gamma} \sum_{k \in \Gamma} \|\mathcal{G}_k\|_{\mathcal{W},S_p},
\end{equation*}
where $\Gamma$ is the set of smoothness directions with cardinality $\gamma$, and $\|.\|_{\mathcal{W},S_p}$ denotes the weighted Schatten $p$-norm.
\end{definition}
 Following the TWCTV definition, the set $\Gamma$ varies by data type: Color images use $\Gamma=\{1,2\}$ ($\gamma=2$) for spatial smoothness, while HSI and color videos use $\Gamma=\{1,2,3\}$ ($\gamma=3$) to include spectral/temporal smoothness. 
 The proposed tensor weighted Schatten $p$-norm is defined on transform-domain slices, and the
corresponding weights are constructed accordingly. Let
$\widehat{\mathcal{G}}_k = \mathcal{G}_k \times_3 M$ denote the tensor
obtained by applying the invertible linear transform
$M \in \mathbb{R}^{n_3 \times n_3}$ along the third mode.
For each $k \in \Gamma$ and frontal slice index $l$,
consider the singular value decomposition
\[
\widehat{\mathcal{G}}_k^{(l)}
= U^{(l)} \, \hat{\mathcal{S}}^{(l)} \, V^{(l)T},
\]
where $\hat{\mathcal{S}}^{(l)} = \operatorname{diag}
(\sigma_1^{(l)}, \ldots, \sigma_{n_{\min}}^{(l)})$
contains the singular values of the $l$-th frontal slice
$\widehat{\mathcal{G}}_k^{(l)}$, and
$n_{\min} = \min\{n_1, n_2\}$.

To construct adaptive weights, the singular values are first normalized.
Define the normalized singular values as
\[
\mathcal{S}_1(i,i,l)
=
\frac{\sigma_i^{(l)} \, m}
{\max_j \sigma_j^{(l)}},
\quad i = 1, \ldots, n_{\min},
\]
where $m>0$ is a scaling parameter controlling the steepness of the
sigmoid function. Based on these normalized values, the weight tensor
$\mathcal{W}$ is defined as an $f$-diagonal tensor whose diagonal entries
are given by the sigmoid function
\begin{equation}
\mathcal{W}(i,i,l)
=
\frac{1}{
1 + \exp\!\big(-\mathcal{S}_1(n_{\min}-i+1,\, n_{\min}-i+1,\, l)\big)
}.
\end{equation}
Here, $\mathcal{W}(i,i,l)$ corresponds to the weight associated with the
$i$-th singular value of the $l$-th transform-domain slice.
The reversal of the index $(n_{\min}-i+1)$ assigns larger weights to
larger singular values and smaller weights to smaller ones, promoting a
nonuniform penalization consistent with low-rank structure. The weights are determined in the transform domain and inversely follow an exponential dependence on the singular values. Due to its exponential form, this weight structure exhibits a faster decay than the commonly used log-sum weights \cite{qinwenjin2024, Chen2020}, enabling more effective discrimination between significant and noise-related singular values. The normalization of singular values ensures scale-invariant regularization, making the weighting scheme robust across different scales of data.\\

The ${M}$-CTV formulation has been proposed to explore low-rankness and smoothness together, but it is possible to establish a rigorous relationship between its weighted Schatten-$p$ regularization and the traditional total variation (TV) measure. The following proposition establishes a quantitative bound between the proposed TWCTV semi-norm and the conventional tensor total variation (TV) norm.
\begin{proposition}\label{prop:TWCTVbound}
Let $\mc{X}\in\mathbb{R}^{n_1\times\cdots\times n_d}$ and define its gradient tensors 
{$\mc{G}_k=\nabla_k(\mc{X})= \mathcal{X} \times_k D_{n_k}$} with invertible transform representations $\hat{\mc{G}}_k=\mathfrak{L}(\mc{G}_k)$. 
For each $k\in\Gamma$ and slice $l$, let $\hat{\mathcal{S}}(i,i,l)=\sigma_i^{(l)}$ $(i=1,\dots,n_{min})$ denote the singular values of 
$\widehat{\mc{G}}_k^{(l)}$, and assume the weights satisfy 
$0<w_{\min}\le \mc{W}(i,i,l)\le w_{\max}\le1$. 
Then, for $0<p<1$, the TWCTV seminorm satisfies
\begin{equation}\label{eq:tv_twctv_relation}
\|\mc{X}\|_{\mathrm{TV}}
\;\le\;
\|\mc{X}\|_{\mathrm{TWCTV}}
\;\le\;
C_p\sqrt{R}\,\|\mc{X}\|_{\mathrm{TV}},
\end{equation}
where $R$ is the tubal rank of $\mc{X}$ and 
$C_p=c^{\frac{1}{2p}}w_{\min}^{-\frac{1}{p}}N^{1-\frac{1}{p}}$ depends on 
the transform normalization $c$, the minimum weight $w_{\min}$, and the number of frontal slices $N=n_3\cdots n_d$.
\end{proposition}

\begin{proof}
Let $s_{k,l}=(\tfrac{1}{\sqrt{c}}\sum_{i=1}^{n_{min}}\mc{W}(i,i,l)\sigma_i^{(l)p})^{1/p}$ 
denote the slicewise contribution to $\|\mc{G}_k\|_{\mc{W},S_p}$.  
Since $w_{\min}\le \mc{W}(i,i,l)\le w_{\max}$, we have
\[
c^{-\frac{1}{2p}}w_{\min}^{\frac{1}{p}}\|\widehat{\mc{G}}_k^{(l)}\|_{S_p}
\le s_{k,l}
\le
c^{-\frac{1}{2p}}w_{\max}^{\frac{1}{p}}\|\widehat{\mc{G}}_k^{(l)}\|_{S_p}.
\]
Using the standard $\ell_p$–$\ell_2$ relation for $0<p<2$,
\[
\|\boldsymbol\sigma^{(l)}\|_2 \le \|\boldsymbol\sigma^{(l)}\|_p 
\le n_{min}^{\frac{1}{p}-\frac{1}{2}}\|\boldsymbol\sigma^{(l)}\|_2,
\]
it follows that
\begin{equation}\label{eq:s_bound}
c^{-\frac{1}{2p}}w_{\min}^{\frac{1}{p}}\|\widehat{\mc{G}}_k^{(l)}\|_F
\le s_{k,l}
\le
c^{-\frac{1}{2p}}w_{\max}^{\frac{1}{p}}n_{min}^{\frac{1}{p}-\frac{1}{2}}\|\widehat{\mc{G}}_k^{(l)}\|_F.
\end{equation}
Combining slices results in
\[
\|\mc{G}_k\|_{\mc{W},S_p}
=\Big(\sum_{l=1}^{N}s_{k,l}^p\Big)^{\frac{1}{p}}
\ge
c^{-\frac{1}{2p}}w_{\min}^{\frac{1}{p}}\Big(\sum_{l=1}^{N}\|\widehat{\mc{G}}_k^{(l)}\|_F^p\Big)^{\frac{1}{p}}.
\]
Since $\|\widehat{\mc{G}}_k\|_F^2=c\,\|\mc{G}_k\|_F^2$ and 
$\|\mc{G}_k\|_F\le\|\mc{G}_k\|_1$, we obtain
\[
\|\mc{X}\|_{\mathrm{TWCTV}}
=\tfrac{1}{\gamma}\sum_{k\in\Gamma}\|\mc{G}_k\|_{\mc{W},S_p}
\ge
c^{-\frac{1}{2p}}w_{\min}^{\frac{1}{p}}\tfrac{1}{\gamma}\sum_{k\in\Gamma}\|\mc{G}_k\|_F
\ge
A_p\|\mc{X}\|_{\mathrm{TV}},
\]
establishing the left inequality in~\eqref{eq:tv_twctv_relation} with $A_p=c^{-\frac{1}{2p}}w_{\min}^{1/p}$.

For the upper bound, from~\eqref{eq:s_bound} we have
$s_{k,l}\le c^{-\frac{1}{2p}}w_{\max}^{\frac{1}{p}}n_{min}^{\frac{1}{p}-\frac{1}{2}}\|\widehat{\mc{G}}_k^{(l)}\|_F$.
Applying the $\ell_p$–$\ell_1$ relation 
$\sum_l s_{k,l}\le N^{1-\frac{1}{p}}(\sum_l s_{k,l}^p)^{1/p}$ 
and using $n_{min}\le R$ for all $l$ yields
\[
\sum_{l=1}^N\|\widehat{\mc{G}}_k^{(l)}\|_F
\le 
c^{\frac{1}{2p}}w_{\min}^{-\frac{1}{p}}N^{1-\frac{1}{p}}R^{\frac{1}{p}-\frac{1}{2}}
\|\mc{G}_k\|_{\mc{W},S_p}.
\]
Hence,
\[
\|\mc{G}_k\|_1
\le \sqrt{n_1n_2}\sum_{l=1}^N\|\mc{G}_k^{(l)}\|_F
\le 
C_p\sqrt{R}\,\|\mc{G}_k\|_{\mc{W},S_p},
\]
with $C_p=c^{\frac{1}{2p}}w_{\min}^{-\frac{1}{p}}N^{1-\frac{1}{p}}$. 
Summing over $k$ and dividing by $\gamma$ gives the right-hand inequality of~\eqref{eq:tv_twctv_relation}.
\end{proof}

This part presents how singular values are regulated during optimization using thresholding operations. First, we use a key lemma on generalized soft thresholding that forms the basis of our singular value thresholding scheme.
\begin{lemma} \cite{hu25convergence} 
For given $p \in (0,1)$ and $w > 0$, the optimization problem:
\begin{equation}
\min_x w|x|^p + \frac{1}{2}(x-y)^2
\end{equation}
is solved by the generalized soft-thresholding (GST) operator:
\begin{equation}\label{eq:gst}
\hat{x} = \text{GST}(y, w, p) = 
\begin{cases}
0 & \text{if }\; |y| \leq \delta, \\
\text{sign}(y)x^* & \text{if }\; |y| > \delta,
\end{cases}
\end{equation}
where $\delta = [2w(1-p)]^{\frac{1}{2-p}} + wp[2w(1-p)]^{\frac{p-1}{2-p}}$ is the threshold value, $\text{sign}(x)$ denotes the signum function, and $x^*$ can be obtained by solving $x^* - |y| + wp(x^*)^{p-1} = 0$ for $x^* > 0$.
\end{lemma}

The following theorem introduces the generalized tensor singular value thresholding (GTSVT) operator based on $M$-product, and establishes how it shrinks singular values during optimization.

\begin{theorem}\label{thm:proxsnn}
Let \(\mathcal{X} \in \mathbb{R}^{n_1 \times n_2 \times \dots \times n_d}\) be any tensor with the $M$-SVD 
$\mathcal{X} = \mathcal{U}*_M\mathcal{S}*_M\mathcal{V}^T,$ and \(\mathcal{W} \in \mathbb{R}^{n_1 \times n_2 \times \cdots \times n_3\cdots n_d}\) be a weight tensor composed of an order-3 f-diagonal tensor. For any \(\tau > 0\), \(0 < p < 1\), the generalized tensor singular value thresholding (GTSVT) operator is defined as
\begin{equation*}
    \mathfrak{D}_{\mathcal{W},\tau}(\mathcal{X}) = \mathcal{U} *_M \mathcal{S}_{\mathcal{W},\tau} *_M \mathcal{V}^T,
\end{equation*}
where
\begin{equation*}
   \mathcal{S}_{\mathcal{W},\tau} =  \big( \text{GST}(\mathcal{S}\times_3 M, \tau \mathcal{W}, p) \big)\times_3 M^{-1}, 
\end{equation*}
 and GST(generalized soft thresholding) is the element-wise shrinkage operator~\eqref{eq:gst}. Then, the GTSVT operator solves the following optimization problem:
\begin{equation*}
    \mathfrak{D}_{\mathcal{W},\tau}(\mathcal{X}) = \arg\min_{\mathcal{Y}} \left\{ \frac{1}{2} \|\mathcal{Y} - \mathcal{X}\|_F^2 + \tau \|\mathcal{Y}\|_{\mc{W},S_p}^p \right\}.
\end{equation*}

\end{theorem}
\begin{proof}
With the invertible transform $\mathfrak{L}$, the optimization problem can be rewritten equivalently as
\begin{equation}
\arg\min_{\mathcal{Y}} \sum_{l=1}^{n_3\cdots n_d}
\left\{ \frac{1}{2} \left\|\hat{\mathcal{Y}}^{(l)} - \hat{\mathcal{X}}^{(l)}\right\|_{F}^{2}
+ \tau \left\|\hat{\mathcal{Y}}^{(l)}\right\|_{\mathcal{W}^{(l)},S_p}^{p} \right\},
\end{equation}
where $\hat{\mathcal{Y}}^{(l)}$ and $\hat{\mathcal{X}}^{(l)}$ represent the $l$-th frontal slice in the transform domain. Since each transformed slice is independent, the optimization problem can be split into $n_3\cdots n_d$ independent subproblems. Each subproblem can be viewed as \begin{equation*}
    \arg\min_{\mathcal{Y}^{(l)}} \{ \frac{1}{2} \|\hat{\mathcal{Y}}^{(l)} - \hat{\mathcal{X}}^{(l)}\|_{F}^2 + \tau \|\hat{\mathcal{Y}}^{(l)}\|_{\mc{W}^{(l)},S_p}^p \},
\end{equation*} which can be effectively solved using an element-wise GST operator. The global solution \(\mathcal{X}\) is computed slice-by-slice in the transform domain, and the inverse transform \(\mathfrak{L}^{-1}\) is then applied to obtain the result in the original domain.
Therefore, the GTSVT operator $\mathfrak{D}_{\mathcal{W},\tau}(\mathcal{X})$ provides the global optimum for the original optimization problem.
\end{proof}

The following corollary introduces a connection between the weighted Schatten p-norm minimization and tensor dictionary learning for group sparse representation, which generalizes Corollary 1 of \cite{zha2020}. This highlights its effectiveness in improving the sparse recovery and low-rank approximation.
\begin{corollary}
    The tensor weighted Schatten-$p$ quasi-norm minimization  \begin{equation*}
        \min_{\mathcal{Y}} \left\{\frac{1}{2}\|\mathcal{X} - \mathcal{Y}\|_F^2 + \|\mathcal{Y}\|_{\mathcal{W},S_p}^p\right\} 
    \end{equation*}can be reformulated as a weighted group $\ell_p$ norm tensor dictionary learning problem:
\begin{equation}
\min_{(\mathcal{D},\mathcal{B})} \left\{\frac{1}{2}\|\mathcal{X} - \mathcal{D} *_M \mathcal{B}\|_F^2 + \|\mathcal{B}\|_{\mathcal{W}',\ell_p}^p\right\},
\end{equation}
where $\mathcal{D} = \mathcal{U}$ and $\mathcal{B} = \mathcal{S} *_M \mathcal{V}^T$. Here $\mathcal{S}$ is a f-diagonal tensor and the weights satisfy the relation:
\begin{equation*}
\mathcal{W}'(i,i,l) = \mathcal{W}(i,i,l)^{1/p}, 1 \leq i \leq \min\{n_1,n_2\}, 1 \leq l \leq n_3\cdots n_d.
\end{equation*}
\end{corollary}
  The complete process of singular value thresholding based on weighted Schatten-$p$ norm is outlined in Algorithm~\ref{alg:GTSVT}.

\begin{algorithm}[htb]
\caption{Generalized tensor singular value thresholding }
\label{alg:GTSVT}
\begin{algorithmic}[1]
    \State \textbf{Input:} $\mathcal{X} \in \mathbb{R}^{n_1 \times n_2 \times \cdots \times n_d}$, $p \in (0,1)$, $\tau > 0$, invertible matrices $M_k$, and weight tensor $\mathcal{W}$.
    \State \textbf{Compute} $\hat{\mathcal{X}} = \mathfrak{L}(\mathcal{X})=\mc{X}\times_3 M_3 \times_4 M_4 \times_5 \cdots \times_d M_d$.
        \For {$l = 1, 2, \dots, n_3\cdots n_d$}
            \State $[\hat{\mathcal{U}}^{(l)}, \hat{\mathcal{S}}^{(l)}, \hat{\mathcal{V}}^{(l)}] = \operatorname{SVD}(\hat{\mathcal{X}}^{(l)})$;
            \State $\mathbf{w} = \operatorname{diag}(\mathcal{W}^{(l)});$
            \For {$i = 1$ to $\min(n_1,n_2)$}
                \State $y = \hat{\mathcal{S}}(i,i,l);$
                \State $\delta = [2\tau w_i(1-p)]^{\frac{1}{2-p}} + \tau w_ip[2\tau w_i(1-p)]^{\frac{p-1}{2-p}};$
                \If {$|y| \leq \delta$}
                    \State $\hat{\mathcal{S}}(i,i,l) = 0;$
                \Else
                    \State Update $x^*$ by solving $x - |y| + \tau w_i p(x)^{p-1} = 0$ (use fixed-point iteration)
                    \State $\hat{\mathcal{S}}(i,i,l) = \text{sign}(y)x^*;$
                \EndIf
            \EndFor
            \State ${\mathcal{C}}^{(l)} = \hat{\mathcal{U}}^{(l)}  \hat{\mathcal{S}}^{(l)}  (\hat{\mathcal{V}}^{(l)})^{T};$
        \EndFor
       
    \State \textbf{Compute} $\mathfrak{D}_{\mathcal{W},\tau}(\mathcal{X}) = \mathfrak{L}^{-1}(\mathcal{C})=\mc{C}\times_3 M_3^{-1} \times_4 M_4^{-1} \times_5 \cdots \times_d M_d^{-1}$.
    \State \textbf{Output:} $\mathfrak{D}_{\mathcal{W},\tau}(\mathcal{X}).$
\end{algorithmic}
\end{algorithm}

We now present our sparse regularization approach using weighted $\ell_1$ norms, where the key challenge lies in the weight selection. Rather than using fixed weights that may ignore the underlying data structure and yield suboptimal solutions, we employed an adaptive approach to learn the weights.
\begin{lemma}\cite{yuan09}\label{lemma:conj}
There exists a convex function $\varphi: \mathbb{R} \rightarrow \mathbb{R}$ such that the function $g(z,\sigma)=\exp(-\frac{\|z\|^2}{\sigma^2})$ admits the representation
\[
g(z, \sigma) = \sup_{p \in \mathbb{R}^-} \left( p \cdot \frac{\|z\|^2}{\sigma^2} - \varphi(p) \right),
\]
where for each fixed $z$, the supremum is achieved at $p = -g(z, \sigma)$.
\end{lemma}

 The following proposition presents a key result for introducing non-convex sparsity-inducing penalties into optimization problems through conjugate function theory. 
\begin{proposition}\label{eqn:conjugate}
The function \( f_{\eta}(x) = \eta(1 - \exp(-|x|/\eta)) \) with $\eta>0$  has a convex conjugate function \(\phi\) such that  
\begin{equation*}
   f_{\eta}(x) = \min_{w \in \mathbb{R}_{+}} \big(w|x| + \phi(w)\big), 
\end{equation*}
and for any fixed value of $x$, the minimizer occurs at $w = \exp(-|x|/\eta)$.

\end{proposition}
\begin{proof}
Let \( |x| = \|z\|^2 \) and \( \eta = \sigma^2 \). Then, from Lemma~\ref{lemma:conj} , we define \( g(z,\sigma) = \exp(-|x| / \eta) \). It follows that
\begin{align*}
f_{\eta}(x) &=\eta(1-g(z,\sigma))\\&= \eta \left(1 - \sup_{p \in \mathbb{R}_{-}} \left(p \frac{|x|}{\eta} - \varphi(p)\right)\right) \\
&= \inf_{p \in \mathbb{R}_{-}} \left(-p |x| + \eta (1 + \varphi(p))\right),
\end{align*}
where the infimum is obtained at \( p = -g(z,\sigma) = -\exp(-|x|/\eta) \).
For convenience , we define \( w = -p \) and \( \phi(w) = \eta (1 + \varphi(-w)) \). Substituting these definitions, we rewrite \( f_{\eta}(x) \) as
\begin{align*}
f_{\eta}(x) &= \min_{w \in \mathbb{R}_{+}} (w |x| + \phi(w)),
\end{align*}

\noindent where \(\phi(w)\) corresponds to the convex conjugate function. Thus, we can conclude that when $w = \exp(-|x|/\eta)$, a minimum value is achieved for any fixed x.
\end{proof}

The function $f_{\eta}(x) = \eta(1 - \exp(-|x|/\eta))$ is effective for sparsity promotion because it approximates the $\ell_0$ norm through its asymptotic behavior: $f_{\eta}(x) \rightarrow |x|$ as $\eta \rightarrow \infty$, and $f_{\eta}(x) \rightarrow 0$ as $\eta \rightarrow 0$. Furthermore, $f_\eta$ exhibits a faster decay than existing weights based on log-sum \cite{qinwenjin2024} and $l_p$ \cite{kong18} regularization.
Given a tensor $\mathcal{E}$, we define its sparse regularization term according to Proposition \ref{eqn:conjugate} as:
\begin{align}\label{eq:sparse}
\mathfrak{R}(\mathcal{E}) 
&= \sum_{i_1,\dots,i_d} f_{\eta}[\mathcal{E}(i_1,\dots, i_d)] \notag\\
&= \sum_{i=1}^{n_1} \sum_{j=1}^{n_2} \sum_{l=1}^{n_3\cdots n_d} f_{\eta}[\mathcal{E}(i,j,l)] \notag\\
&= \min_{\mathcal{W}_{\mathcal{E}}} \sum_{i=1}^{n_1} \sum_{j=1}^{n_2} \sum_{l=1}^{n_3\cdots n_d} \left( \mathcal{W}_{\mathcal{E}}(i,j,l) |\mathcal{E}(i,j,l)| + \phi[\mathcal{W}_{\mathcal{E}}(i,j,l)] \right).
\end{align}

\noindent The weight tensor $\mathcal{W}_{\mc{E}}$ is updated by:
\begin{equation*}
\mathcal{W}_{\mc{E}} = G_{\mathcal{E}}(\mathcal{E}),
\end{equation*}
where $G_{\mathcal{E}}: \mathbb{R}^{n_1 \times n_2 \times n_3  \cdots  n_d} \rightarrow \mathbb{R}^{n_1 \times n_2 \times n_3 \cdots  n_d}$ is defined by:$(G_{\mathcal{E}}(\mathcal{E}))(i,j,l) = \exp(-|\mathcal{E}(i,j,l)|/\eta)$. The parameter $\eta$ can be adaptively chosen as $c_{\mathcal{E}} \cdot \text{mean}(\{|\mathcal{E}(i,j,l)|: (i,j,l) \in [n_1] \times [n_2] \times [n_3\cdots n_d]\}),$
where $c_{\mathcal{E}}>0$ is a constant and {$[n] := \{1,2,\ldots,n\}$ denotes the index set of integers from $1$ to $n$.} 
Thus, it provides adaptive weighting that naturally enhances sparsity by applying stronger penalties to smaller values while preserving larger ones.

\section{Optimization Models}

This section presents the models for RLRTC using the nonconvex regularizer defined in the previous section. Let $\mathcal{M} \in \mathbb{R}^{n_1 \times n_2 \times \dots \times n_d}$ represent the observed tensor possessing global low-rank and local smoothness properties. Following common practice in tensor completion, we employ Bernoulli random sampling, $\Omega \sim \text{Ber}(p)$ to observe entries. By integrating a nonconvex regularizer and modeling the corruption explicitly, the RLRTC problem is formulated as:
\begin{equation}\label{eq:RTCmodel}
\min_{\mathcal{X},\mathcal{E}} \frac{1}{\gamma} \sum_{k\in\Gamma} \|\mathcal{G}_k\|_{\mathcal{W},S_p} + \lambda \, \mathfrak{R}(\mathcal{E}) \quad \text{subject to} \quad P_{\Omega}(\mathcal{X} + \mathcal{E}) = P_{\Omega}(\mathcal{M}),
\end{equation}
where $P_{\Omega}(\cdot)$ is the projection operator onto the observed entries defined by the index set $\Omega$, satisfying the idempotent property $P_{\Omega}^2 = P_{\Omega}$. $\|\cdot\|_{\mathcal{W},S_p}$ denotes a nonconvex weighted Schatten-$p$ norm applied to gradient tensors $\mathcal{G}_k$, and $\mathfrak{R}(\mathcal{E})$ is a sparse regularizer (e.g., weighted $\ell_1$ norm). This model reduces to standard tensor completion (TC) when $\mathcal{E} = 0$, and to tensor robust PCA (TRPCA) when $\Omega$ is the entire index set.

\begin{figure}[ht!]
\centering
\begin{tikzpicture}[
  x=1cm, y=1cm,
  tensor/.style={draw, thick, fill=cyan!40, minimum width=1.8cm, minimum height=1.8cm},
  sparse/.style={draw, thick, pattern=dots, minimum width=1.8cm, minimum height=1.8cm},
  arrow/.style={->, thick, >=stealth, shorten >=2pt, shorten <=2pt},
  dashedbox/.style={draw, dashed, thick, rounded corners},
  every node/.style={inner sep=1pt} 
]
\path[use as bounding box] (-8,-6.5) rectangle (8,6);

\node[anchor=center] (img1) at (-6.5, 3.5) {\includegraphics[width=2cm]{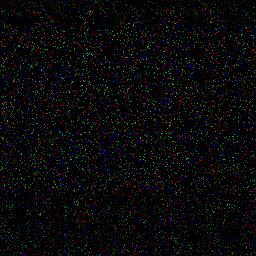}};
\node[below=0.15cm of img1, color=blue1, font=\small] {Incomplete color image};

\node[tensor, inner sep=0pt] (lowrank1) at (1.5, 3.5) {};
\draw[thick, fill=cyan!40] (0.6, 2.6) -- (0.6, 4.4) -- (2.4, 4.4) -- (2.4, 2.6) -- cycle;
\draw[thick, fill=cyan!30] (2.4, 4.4) -- (2.8, 4.8) -- (2.8, 3.0) -- (2.4, 2.6);
\draw[thick, fill=cyan!20] (0.6, 4.4) -- (1.0, 4.8) -- (2.8, 4.8) -- (2.4, 4.4);
\node[font=\Large] at (lowrank1.center) {$\mathcal{X}$};
\node[below=0.3cm of lowrank1, color=blue1, font=\small] {Low-rank tensor};

\node[anchor=center] (out1) at (6.5, 3.5) {\includegraphics[width=2.2cm]{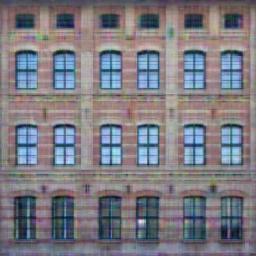}};
\node[below=0.1cm of out1, color=blue1, font=\small] {Completion image};

\node[anchor=center] (img2) at (-6.5, 0) {\includegraphics[width=2cm]{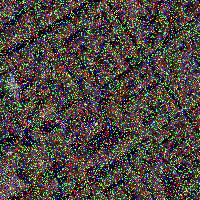}};
\node[below=0.25cm of img2, color=blue1, font=\small] {Noisy hyperspectral image};

\node[tensor, pattern=dots, inner sep=0pt] (observed) at (-3, 0) {};
\draw[thick, fill=cyan!40] (-3.9, -0.9) -- (-3.9, 0.9) -- (-2.1, 0.9) -- (-2.1, -0.9) -- cycle;
\draw[thick, fill=cyan!30] (-2.1, 0.9) -- (-1.7, 1.3) -- (-1.7, -0.5) -- (-2.1, -0.9);
\draw[thick, fill=cyan!20] (-3.9, 0.9) -- (-3.5, 1.3) -- (-1.7, 1.3) -- (-2.1, 0.9);
\node[font=\Large] at (observed.center) {$\mathcal{M}$};

\pgfmathsetseed{42}

\foreach \i in {1,...,35} {
  \pgfmathsetmacro{\randx}{-3.85 + rnd*1.7}  
  \pgfmathsetmacro{\randy}{-0.85 + rnd*1.7}
  \fill (\randx, \randy) circle (0.05);
}

\foreach \i in {1,...,20} {
  \pgfmathsetmacro{\u}{rnd}  
  \pgfmathsetmacro{\v}{rnd}  
  \pgfmathsetmacro{\dotx}{-3.9 + \u*1.8 + \v*0.4}
  \pgfmathsetmacro{\doty}{0.9 + \v*0.4}
  \fill (\dotx, \doty) circle (0.05);
}

\foreach \i in {1,...,20} {
  \pgfmathsetmacro{\u}{rnd}  
  \pgfmathsetmacro{\v}{rnd}  
  \pgfmathsetmacro{\dotx}{-2.1 + \v*0.4}
  \pgfmathsetmacro{\doty}{-0.9 + \u*1.8 + \v*0.4}
  \fill (\dotx, \doty) circle (0.05);
}

\node[font=\Large] at (observed.center) {$\mathcal{M}$};
\node[below=0.35cm of observed, color=blue1, font=\small] {Observed tensor};

\node[tensor, inner sep=0pt] (lowrank2) at (1.5, 0) {};
\draw[thick, fill=cyan!40] (0.6, -0.9) -- (0.6, 0.9) -- (2.4, 0.9) -- (2.4, -0.9) -- cycle;
\draw[thick, fill=cyan!30] (2.4, 0.9) -- (2.8, 1.3) -- (2.8, -0.5) -- (2.4, -0.9);
\draw[thick, fill=cyan!20] (0.6, 0.9) -- (1.0, 1.3) -- (2.8, 1.3) -- (2.4, 0.9);
\node[font=\Large] at (lowrank2.center) {$\mathcal{X}$};
\node[below=0.3cm of lowrank2, color=blue1, font=\small] {Low-rank tensor};

\node[anchor=center] (out2) at (6.5, 0) {\includegraphics[width=2.2cm]{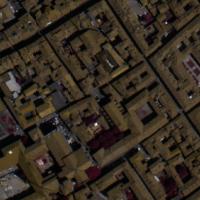}};
\node[below=0.1cm of out2, color=blue1, font=\small] {Recovered image};

\node[anchor=center, opacity=1] (frame4) at (-6.9, -3.2) {\includegraphics[width=2cm, height=1cm]{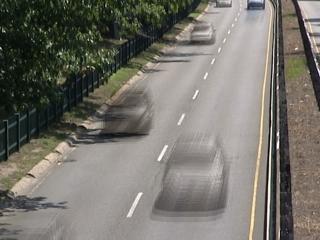}};
\node[anchor=center, opacity=1] (frame3) at (-6.7, -3.5) {\includegraphics[width=2cm, height=1cm]{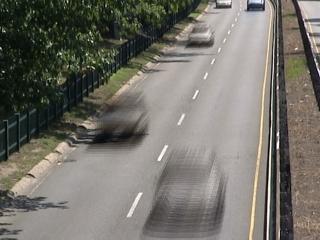}};
\node[anchor=center, opacity=1] (frame2) at (-6.5, -3.8) {\includegraphics[width=2cm, height=1cm]{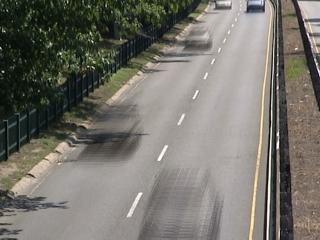}};
\node[anchor=center, opacity=1.0] (frame1) at (-6.3, -4.1) {\includegraphics[width=2cm, height=1cm]{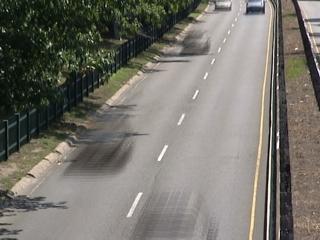}};
\node[below=0.2cm of frame1, color=blue1, font=\small] {Color video frames};

\node[sparse, inner sep=0pt] (sparse) at (1.5, -3.5) {};
\draw[thick, fill=cyan!40] (0.6, -4.4) -- (0.6, -2.6) -- (2.4, -2.6) -- (2.4, -4.4) -- cycle;
\draw[thick, fill=cyan!30] (2.4, -2.6) -- (2.8, -2.2) -- (2.8, -4.0) -- (2.4, -4.4);
\draw[thick, fill=cyan!20] (0.6, -2.6) -- (1.0, -2.2) -- (2.8, -2.2) -- (2.4, -2.6);

\pgfmathsetseed{42}  

\foreach \i in {1,...,25} {  
  \pgfmathsetmacro{\randx}{0.65 + rnd*1.7}
  \pgfmathsetmacro{\randy}{-4.35 + rnd*1.7}
  \fill (\randx, \randy) circle (0.05);
}

\foreach \i in {1,...,15} {
  \pgfmathsetmacro{\u}{rnd}
  \pgfmathsetmacro{\v}{rnd}
  \pgfmathsetmacro{\dotx}{0.6 + \u*1.8 + \v*0.4}
  \pgfmathsetmacro{\doty}{-2.6 + \v*0.4}
  \fill (\dotx, \doty) circle (0.05);
}

\foreach \i in {1,...,15} {
  \pgfmathsetmacro{\u}{rnd}
  \pgfmathsetmacro{\v}{rnd}
  \pgfmathsetmacro{\dotx}{2.4 + \v*0.4}
  \pgfmathsetmacro{\doty}{-4.4 + \u*1.8 + \v*0.4}
  \fill (\dotx, \doty) circle (0.05);
}

\node[font=\Large] at (sparse.center) {$\mathcal{E}$};
\node[below=0.3cm of sparse, color=blue1, font=\small] {Sparse tensor};

\node[anchor=center, opacity=1] (fg4) at (6.1, -3.2) {\includegraphics[width=2cm, height=1cm]{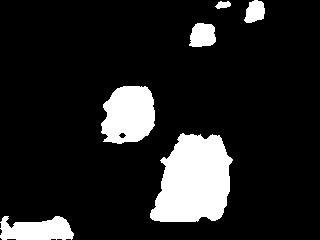}};
\node[anchor=center, opacity=1] (fg3) at (6.3, -3.5) {\includegraphics[width=2cm, height=1cm]{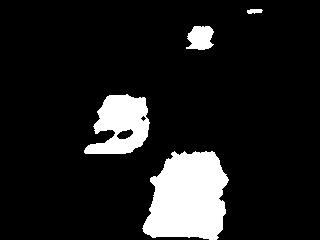}};
\node[anchor=center, opacity=1] (fg2) at (6.5, -3.8) {\includegraphics[width=2cm, height=1cm]{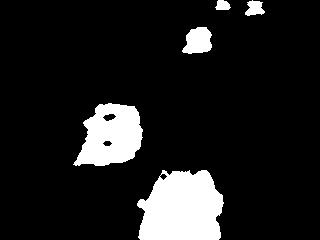}};
\node[anchor=center, opacity=1.0] (fg1) at (6.7, -4.1) {\includegraphics[width=2cm, height=1cm]{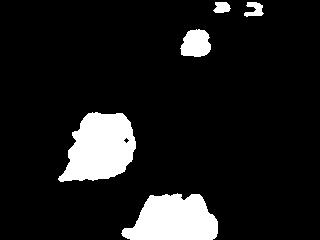}};
\node[below=0.2cm of fg1, color=blue1, font=\small] {Foreground};


\node[dashedbox, minimum width=3cm, minimum height=7cm] (rpca_box) at (1.5, -1.75) {};
\node at (1.5, -1.75) {\Large $+$};


\draw[arrow] (img1.east) -- ++(2.4,0) -- (observed.north);

\draw[arrow] (observed.east) -- (lowrank1.west);

\node[rotate=45, color=blue1] (N) at (-.7,2.7) {Tensor completion}; 

\draw[arrow] (lowrank1.east) -- (out1.west);

\draw[arrow] (img2.east) -- (observed.west);

\draw[arrow] (observed.east) -- (rpca_box.west);
\node[rotate=-45, color=blue1] (N) at (-1,-1.5) {Tensor RPCA}; 

\draw[arrow] (lowrank2.east) -- (out2.west);

\draw[arrow] (frame2.east) -- ++(2.4,0) -- (observed.south);

\draw[arrow] (sparse.east) -- (fg2.west);

\end{tikzpicture}

\caption{Flowchart of the model for tensor completion and TRPCA.}
   \label{fig:flowchart}
\end{figure}
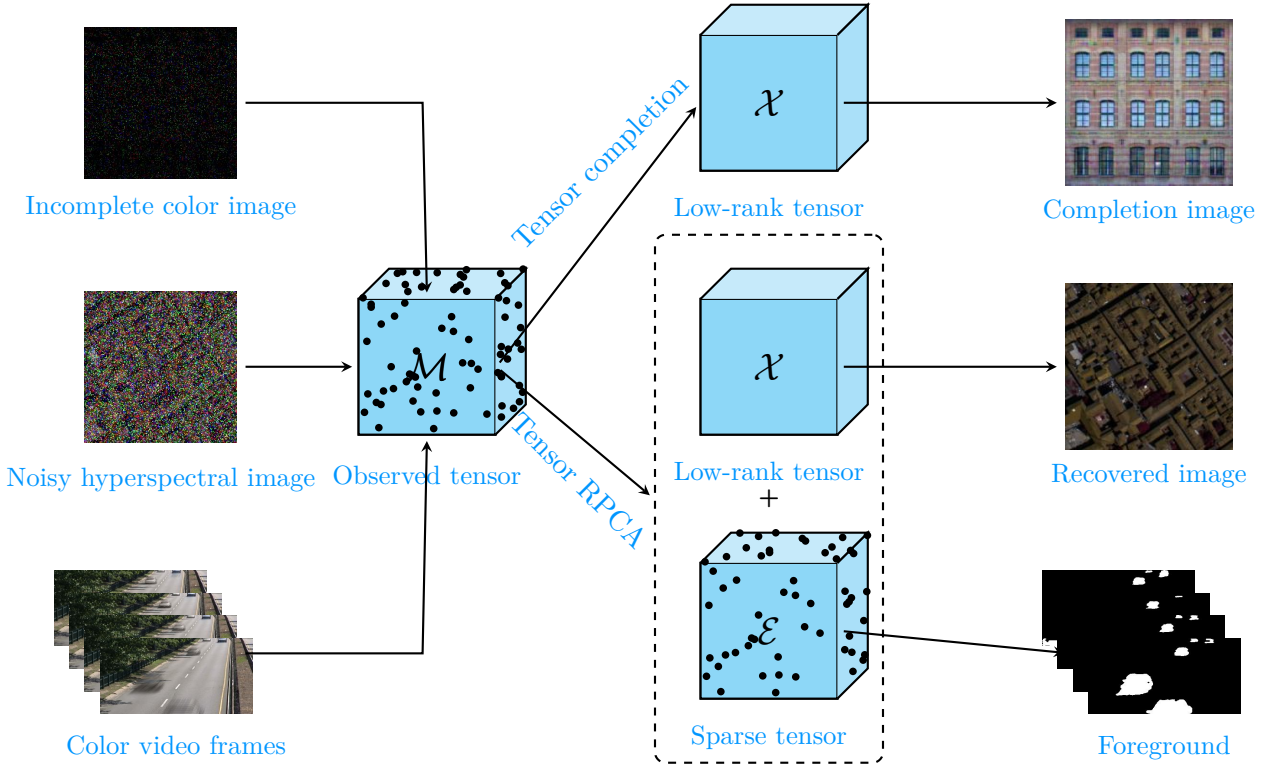
The formulations in Equation~\eqref{eq:RTCmodel} are inherently nonconvex, and the variables are highly coupled, making direct optimization challenging. To overcome these issues, we adopt the ADMM framework from \cite{boyd2011,wenjin2022tnn}. The ADMM effectively handles nonconvex problems by decomposing them into simpler subproblems that often admit closed-form solutions while maintaining convergence properties under suitable conditions.

\subsection{Robust low rank tensor completion (RLRTC) }
{In this subsection, an ADMM-based optimization framework is presented to solve the proposed nonconvex model~\eqref{eq:RTCmodel}:
\begin{equation}\label{eq:trpca1}
\min_{\mathcal{X},\mathcal{E}} \frac{1}{\gamma}\sum_{k=1}^{\gamma}\|\mathcal{G}_k\|_{\mc{W},S_p} + \lambda \mathfrak{R}(\mathcal{E}) \quad \text{subject to} \quad P_{\Omega}(\mathcal{M}) =\mathcal{X} + \mathcal{E},~\mathcal{G}_k = \nabla_k(\mathcal{X}), ~k \in\Gamma,
\end{equation}
where $\mathcal{M} = P_{\Omega}(\mathcal{X} + \mathcal{E})$ denotes the observed tensor on the sampling set $\Omega$, $\mathcal{X}$ is the unknown clean tensor, and $\mathcal{E}$ represents the noise or sparse error component.}
Also note that when the noise tensor $\mathcal{E} = 0$, we replace the term $\mathcal{E}$ with an auxiliary variable $\mathcal{K}$, and the regularization term $\lambda \, \mathfrak{R}(\mathcal{E})$ with $\delta_{\Omega}(\mathcal{K})$. Here, $\mathcal{K}$ models the absence of entries in $\Omega^{\perp}$ using the indicator function $\delta_{\Omega}(\cdot)$, which is defined as
\[
\delta_{\Omega}(\mathcal{K}) =
\begin{cases}
0, & \text{if } P_{\Omega}(\mathcal{K}) = 0, \\
+\infty, & \text{otherwise}.
\end{cases}
\]

\noindent Now, for the weighted $\ell_1$ norm used in sparse regularization of $\mathcal{E}$, denoted as $\|\mathcal{E}\|_{\mathcal{W}_{\mathcal{E}},1}$, we have:

\begin{equation}
\|\mathcal{E}\|_{\mc{W}_{\mc{E}},1} = \min_{\mc{W}_{\mc{E}}} \sum_{i=1}^{n_1}\sum_{j=1}^{n_2}\sum_{l=1}^{n_3\cdots n_d}(\mathcal{W}_{\mc{E}}(i,j,l)|\mathcal{E}(i,j,l)| + \phi(\mathcal{W}_{\mc{E}}(i,j,l))),
\end{equation}
where the minimum is attained at $\mathcal{W}_{\mc{E}}(i,j,l) = \exp(-|\mathcal{E}(i,j,l)|/\eta)$.
Thus, by integrating these equations and combining the definition of weighted schatten p-norm and weighted $\ell_1$, the problem~\eqref{eq:trpca1} can be reformulated as:
\begin{equation}\label{eq:trpca3}
\begin{aligned}
\min_{\mathcal{X},\mathcal{E},\mathcal{G}_k,\mathcal{W}_{\mc{E}}}  \frac{1}{\gamma}\sum_{k\in\Gamma} \|\mathcal{G}_k\|_{\mc{W},S_p} + \lambda\|\mathcal{E}\|_{\mathcal{W}_{\mc{E}},1} 
+ \phi(\mathcal{W}_{\mc{E}}) \quad \text{subject to} \quad P_{\Omega}(\mathcal{M}) = \mathcal{X} + \mathcal{E}, \mathcal{G}_k = \nabla_k(\mathcal{X}),
\end{aligned}
\end{equation}
where $\phi(\mathcal{W})$ is specified such that $(\phi(\mathcal{W}_{\mc{E}}))(i,j,l) = \phi(\mathcal{W}_{\mc{E}}(i,j,l))$.

\noindent Accordingly, the scaled form of the Lagrangian function for Equation~\eqref{eq:trpca3} is expressed as:
\begin{equation}\label{eq:trpca2}
\begin{aligned}
L(\mathcal{X}, \mathcal{G}_k, \mathcal{E}, \mathcal{W}_{\mc{E}}, \mathcal{Y}_k, {\Upsilon}) &= \frac{1}{\gamma}\sum_{k=1}^{\gamma} \|\mathcal{G}_k\|_{\mc{W},S_p} + \lambda\|\mathcal{E}\|_{\mathcal{W}_{\mc{E}},1} + \sum_{k=1}^{\gamma} \frac{\mu}{2}\|\nabla_k(\mathcal{X}) - \mathcal{G}_k + \frac{\mathcal{Y}_k}{\mu}\|_F^2  \\
&+\phi(\mathcal{W}_{\mc{E}})
+ \frac{\mu}{2}\|P_{\Omega}(\mathcal{M}) - \mathcal{X} - \mathcal{E} + \frac{\Upsilon}{\mu}\|_F^2 ,
\end{aligned}
\end{equation}
where ${\Upsilon}$ and $\mathcal{Y}_k$ represent the Lagrange multipliers, and $\mu$ is a positive penalty parameter. Each variable can be updated alternately in the scheme of the ADMM framework.


\noindent\textbf{Step 1:} We derive the $\mathcal{X}$-subproblem to update $\mathcal{X}$ while keeping all other variables fixed in~\eqref{eq:trpca2}:
\begin{equation}\label{eq:trpca31}
\mathcal{X}_{t+1}
=
\arg\min_{\mathcal{X}}
\sum_{k=1}^{\gamma}
\frac{\mu_t}{2}
\left\|
\nabla_k(\mathcal{X})
-
\mathcal{G}_{t(k)}
+
\frac{\mathcal{Y}_{t(k)}}{\mu_t}
\right\|_F^2
+
\frac{\mu_t}{2}
\left\|
P_{\Omega}(\mathcal{M})
-
\mathcal{X}
-
\mathcal{E}_t
+
\frac{\Upsilon_t}{\mu_t}
\right\|_F^2 .
\end{equation}

\noindent Taking the derivative of the objective function in~\eqref{eq:trpca31} with respect to $\mathcal{X}$ and setting it to zero yields
\begin{equation}
P_{\Omega}(\mathcal{M})
-
\mathcal{E}_t
+
\frac{\Upsilon_t}{\mu_t}
+
\sum_{k=1}^{\gamma}
\nabla_k^T
\!\left(
\mathcal{G}_{t(k)}
-
\frac{\mathcal{Y}_{t(k)}}{\mu_t}
\right)
=
\left(
\mathcal{I}
+
\sum_{k=1}^{\gamma}
\nabla_k^T \nabla_k
\right)
(\mathcal{X}),
\end{equation}

\noindent where $\nabla_k^T(\cdot)$ denotes the transpose operator of $\nabla_k(\cdot)$. 
Note that the difference operation on tensors is linear via tensor--tensor product. 
Following~\cite{wang2008newalternating}, we can apply multi-dimensional FFT, which diagonalizes $\nabla_k(\cdot)$'s corresponding difference tensors $\mathcal{D}_k$, enabling efficient computation of the optimal solution of~\eqref{eq:trpca31} based on the convolution theorem of Fourier transforms. Therefore, the closed-form solution for updating $\mathcal{X}$ can be expressed as
\begin{equation*}
\mathcal{X}_{t+1}
=
\mathcal{F}^{-1}
\left(
\frac{
\mathcal{F}\!\left(
P_{\Omega}(\mathcal{M})
-
\mathcal{E}_t
+
\frac{\Upsilon_t}{\mu_t}
+
\mathcal{H}
\right)
}{
\mathbf{1}
+
\sum_{k=1}^{\gamma}
\mathcal{F}(\mathcal{D}_k)^{H}
\odot
\mathcal{F}(\mathcal{D}_k)
}
\right),
\end{equation*}
where $\mathcal{H}
=
\sum_{k=1}^{\gamma}
\nabla_k^T
\!\left(
\mathcal{G}_{t(k)}
-
\frac{\mathcal{Y}_{t(k)}}{\mu_t}
\right).$

Here, $\mathcal{F}$ and $\mathcal{F}^{-1}$ denote the multi-dimensional fast Fourier transform (FFT) and its inverse, respectively, $(\cdot)^{H}$ denotes the conjugate transpose (Hermitian), $\mathbf{1}$ is a tensor with all entries equal to one, and $\odot$ represents the Hadamard (element-wise) product.

\noindent\textbf{Step 2:} We update \(\mathcal{G}_k\) by keeping the other variables fixed:
\begin{equation}\label{eq:trpca41}
\mathcal{G}_{t+1(k)} = \arg\min_{\mathcal{G}_k} \frac{1}{\gamma\mu_t}\sum_{k=1}^{\gamma} \|\mathcal{G}_k\|_{\mc{W},S_p} + \frac{1}{2}\|\nabla_k(\mathcal{X}_{t+1}) - \mathcal{G}_k + \frac{\mathcal{Y}_{t(k)}}{\mu_t}\|_F^2.
\end{equation}
The update for \(\mathcal{G}_{t+1(k)}\) is then obtained using the proximal operator as defined in Theorem~\ref{thm:proxsnn}:
\begin{equation}
\mathcal{G}_{t+1(k)} = \mathfrak{D}_{\mc{W},\frac{1}{\gamma\mu_t}}\left(\nabla_k(\mathcal{X}_{t+1}) + \frac{\mathcal{Y}_{t(k)}}{\mu_t}\right).
\end{equation}

\noindent\textbf{Step 3:} Next, we update $\mathcal{E}$ by fixing the other variables:
\begin{equation}\label{eq:trpca5}
\mathcal{E}_{t+1}
=
\arg\min_{\mathcal{E}}
\frac{1}{2}
\left\|
P_{\Omega}(\mathcal{M})
-
\mathcal{X}_{t+1}
-
\mathcal{E}
+
\frac{\Upsilon_t}{\mu_t}
\right\|_F^2
+
\frac{\lambda}{\mu_t}
\|\mathcal{E}\|_{\mc{W}_{\mc{E}},1}.
\end{equation}

\noindent This problem admits an explicit solution, which can be obtained using the tensor soft-thresholding operator $S$, defined element-wise by
\[
S_\tau(x)
=
\mathrm{sign}(x)\,
\max\!\left(|x|-\tau,\,0\right).
\]

\noindent Accordingly, the $(i,j,l)$-th entry of $\mathcal{E}_{t+1}$ is given by
\[
\mathcal{E}_{t+1}(i,j,l)=
\begin{cases}
S_{\frac{\lambda}{\mu_t}\,\mathcal{W}_{\mathcal{E}(t)}(i,j,l)}
\!\left(
P_{\Omega}(\mathcal{M})(i,j,l)
-
\mathcal{X}_{t+1}(i,j,l)
+
\frac{\Upsilon_t(i,j,l)}{\mu_t}
\right),
& \text{if } (i,j,l)\in\Omega,\\[8pt]
0, & \text{otherwise}.
\end{cases}
\]
Next we update the weight $\mc{W}_\mc{E}$ by $\mc{W}_{\mc{E}(t+1)} =G_\mc{E}(\mc{E}_t).$

\noindent\textbf{Step 4:} 
Also note that when $\mathcal{E} = 0$, Step 3 is omitted, and $\mathcal{E}$ is replaced by an auxiliary variable $\mathcal{K}$ to enforce the data consistency constraint over the observed set. Then we update $\mathcal{K}$ using the equation
\begin{equation}\label{eq:trpca61}
\mathcal{K}_{t+1} = P_{\Omega}(\mathcal{M}) - \mathcal{X}_{t+1} + \frac{\mathcal{Y}_{t(k)}}{\mu}, \quad \text{with} \quad P_{\Omega}(\mathcal{K}_{t+1}) = 0.
\end{equation}

\noindent\textbf{Step 5:} Finally, we update the Lagrange multipliers as
\begin{equation}\label{eq:trpca7}
\begin{aligned}
\mathcal{Y}_{t+1(k)}
&=
\mathcal{Y}_{t(k)}
+
\mu_t
\left(
\nabla_k(\mathcal{X}_{t+1})
-
\mathcal{G}_{t+1(k)}
\right),
\quad k \in \Gamma, \\[6pt]
\Upsilon_{t+1}
&=
\Upsilon_t
+
\mu_t
\left(
P_{\Omega}(\mathcal{M})
-
\mathcal{X}_{t+1}
-
\mathcal{E}_{t+1}
\right).
\end{aligned}
\end{equation}
For low-rank tensor completion, we replace $\mathcal{E}_{t+1}$ with $\mathcal{K}_{t+1}$ in the above formulation.
The complete ADMM optimization steps are presented in Algorithm~\ref{alg:TWCTVTRPCA}.

\begin{algorithm}[htb]

\caption{Algorithm for the TWCTV based RLRTC}
\label{alg:TWCTVTRPCA}
\begin{algorithmic}[1]
    \State \textbf{Input:} Observed tensor $\mathcal{M}$, observation set $\Omega$, priori set $\Gamma$, parameter $\lambda > 0$.
    \State \textbf{Initialize:} $\mathcal{X}_0$, $\mathcal{G}_{0(k)}=\mathcal{E}_{0}=\mathcal{Y}_{0(k)}=\Upsilon_{0}=\mathcal{O}$, $\mc{W}_{\mc{E}}=\bf{1}$, $\rho=1.1$, $\mu_{0}=1 \times 10^{-4}$, $\mu_{\max }=1 \times 10^{10}$, $\epsilon=1 \times 10^{-8}$, $c_{\mathcal{E}}=2$, $t=0$.
    \While{not converged}
        \State Update $\mathcal{X}_{t+1}$ using $\mc{X}$ subproblem \eqref{eq:trpca3};
        \State Update $\mathcal{G}_{t+1(k)}$ ($k\in\Gamma$) by weighted Schatten-$p$ norm based proximal operator used in \eqref{eq:trpca41};
        \State \textbf{if} $\Omega$ is the full index set \textbf{then}
    \State \quad Update $\mathcal{E}_{t+1}$ by solving subproblem~\eqref{eq:trpca5};
    \State \quad Set $\eta = c_{\mathcal{E}} \cdot \text{mean} \left( \left\{ |\mathcal{E}_{t+1}(i,j,l)| : (i,j,l) \in [n_1] \times [n_2] \times [n_3 \cdots n_d] \right\} \right);$
    \State \quad Update weights $\mathcal{W}_{\mathcal{E}(t+1)}$ by: $(\mathcal{W}_{\mathcal{E}(t+1)})(i,j,l) = \exp \left( -\frac{|\mathcal{E}_{t+1}(i,j,l)|}{\eta} \right)$;
    \State \textbf{else if} $\mathcal{E} = 0$ \textbf{then}
    \State \quad Update $\mathcal{K}_{t+1}$ by solving subproblem~\eqref{eq:trpca61};
        \State Update Lagrange multipliers $\mathcal{Y}_{t+1(k)}, \Upsilon_{t+1}$ by~\eqref{eq:trpca7};
           
        \State Update penalty parameter: $\mu_{t+1} = \min(\rho\mu_t, \mu_{\max})$;
             
            \State $t = t + 1$.
        \State \textbf{Check convergence:}
    \State $\|\mathcal{X}_{t+1} - \mathcal{X}_t\|_{\infty} \leq \epsilon$,
    \State $\|\mathcal{E}_{t+1} - \mathcal{E}_t\|_{\infty} \leq \epsilon$ \quad (\textit{or} $\|\mathcal{K}_{t+1} - \mathcal{K}_t\|_{\infty} \leq \epsilon$),
    \State $\|P_{\Omega}(\mathcal{M}) - \mathcal{X}_{t+1} - \mathcal{E}_{t+1}\|_{\infty} \leq \epsilon$ \quad (\textit{or} $\|P_{\Omega}(\mathcal{M}) - \mathcal{X}_{t+1} - \mathcal{K}_{t+1}\|_{\infty} \leq \epsilon$).

    \EndWhile
    \State \textbf{Output:} Recovered tensors $\mathcal{X}$ and $\mathcal{E}$.
\end{algorithmic}
\end{algorithm}
\subsection{Complexity analysis}
 The first step involves computing the transformed representation of the tensor, typically via DCT or any invertible transform along the third mode. This transformation requires a time complexity of $\mathcal{O}(n_1 n_2 n_3\cdots n_d \log(n_3\cdots n_d))$ or $\mathcal{O}(n_1 n_2 (n_3\cdots n_d )^2)$  , as each of the $n_1 n_2$ mode-3 tubes is transformed individually. Following the transformation, the weighted Schatten-$p$ norm minimization is conducted through a slice-wise generalized singular value thresholding (GSVT) operation. For each of the $n_3\cdots n_d$ frontal slices of size $n_1 \times n_2$, this operation involves an SVD and generalised soft-thresholding step, with an overall complexity of $\mathcal{O}(n_{\text{max}} n_{\text{min}}^2 n_3\cdots n_d)$, where $n_{\text{max}} = \max(n_1, n_2)$ and $n_{\text{min}} = \min(n_1, n_2)$. The sparse component $\mathcal{E}$ is updated via a weighted $\ell_1$ norm proximal operator, which admits a closed-form soft-thresholding solution. As it operates elementwise over the entire tensor, its computational cost is $\mathcal{O}(n_1 n_2 n_3 \cdots n_d)$. The update of the gradient tensors $\nabla_k(\mathcal{X})$, the adaptive weight tensor $\mathcal{W}_{\mathcal{E}}$, and the Lagrange multipliers $\mathcal{Y}_k$ and $\Upsilon$ each contribute an additional cost of $\mathcal{O}(n_1 n_2 n_3\cdots n_d)$. Therefore, the total time complexity of a single ADMM iteration for the proposed Algorithm~\ref{alg:TWCTVTRPCA} is 
\[
\mathcal{O}(n_1 n_2 n_3 \cdots n_d \log(n_3 \cdots n_d)) + \mathcal{O}(n_{\text{max}} n_{\text{min}}^2 n_3 \cdots n_d),
\]
when using DCT or DFT (due to fast transform implementation), and 
\[
\mathcal{O}(n_1 n_2 (n_3 \cdots n_d)^2) + \mathcal{O}(n_{\text{max}} n_{\text{min}}^2 n_3 \cdots n_d),
\]
when using any general invertible transform.

\section{ Convergence Analysis }

This section analyzes the convergence behavior of the sequences produced by Algorithm~\ref{alg:TWCTVTRPCA}. We first present several key lemmas concerning the boundedness of these sequences, which form the foundation for convergence analysis.
\begin{lemma}\label{eqn:lemma2}
    The sequences $\{\mc{Y}_{t(k)}\}_{t=1}^{\infty}$,$\{\mc{W}_{\mc{E}(t)}\}_{t=1}^{\infty}$,$\{{\Upsilon}_{t}\}_{t=1}^{\infty}$ generated by the Algorithm \ref{alg:TWCTVTRPCA} are bounded.
\end{lemma}
\begin{proof}
    By definition, the sequence $\{\mc{W}_{\mc{E}(t)}\}_{t=1}^{\infty}$ is bounded by 1.
    Now, we use the properties of tensor Frobenius norm $\|\mc{A}\|_F=\frac{1}{\sqrt{c}}\|\operatorname{bdiag}(\mc{A})\|_F$, where $\operatorname{bdiag}(\mc{A})$ represents a block diagonal matrix:
\begin{equation*}
\operatorname{bdiag}(\mc{A}) = 
\begin{bmatrix}
\mc{A}^{(1)} & 0 & \cdots & 0 \\
0 & \mc{A}^{(2)} & \cdots & 0 \\
\vdots & \vdots & \ddots & \vdots \\
0 & 0 & \cdots & \mc{A}^{(n_3\cdots n_d)}
\end{bmatrix}.
\end{equation*}
Then using the relation $\|\mc{A}\|_F=\frac{1}{\sqrt{c}}\|\operatorname{bdiag}(\hat{\mc{A}})\|_F$ ( for some nonzero $c$) from Theorem $3.1$ in~\cite{kilmer2021}, we have:

\begin{align}
\begin{split}
 \|\mathcal{Y}_{t+1(k)}\|_F &= \mu_t\left\|\mathfrak{D}_{\mathcal{W},\frac{1}{\gamma\mu_t}}(\mc{X}_t) - \mc{X}_t\right\|_F 
 \quad \text{(where } \mc{X}_t = \nabla_k(\mathcal{X}_{t+1}) + \tfrac{\mathcal{Y}_{t(k)}}{\mu_t}) \\
 &= \frac{\mu_t}{\sqrt{c}}\left\|\operatorname{bdiag}\left(\hat{\mathfrak{D}}_{\mathcal{W},\frac{1}{\gamma\mu_t}}(\mc{X}_t)\right) - \operatorname{bdiag}(\hat{\mc{X}}_t)\right\|_F \\
 &= \frac{\mu_t}{\sqrt{c}}\left\|\operatorname{bdiag}\left((\mathcal{U}_t \ast_M (\mathcal{S}_{t,\mathcal{W},\frac{1}{\gamma\mu_t}} - \mathcal{S}_t )\ast_M \mathcal{V}_t^{T})\times_3 M\right)\right\|_F 
 \;\\
 &= \frac{\mu_t}{\sqrt{c}}\|\operatorname{bdiag}(\hat{\mathcal{S}}_{t,\mathcal{W},\frac{1}{\gamma\mu_t}} - \hat{\mathcal{S}}_t )\|_F,
\end{split}
\label{eq:lemma511}
\end{align}
where the SVD of  $\mc{X}_t = \mathcal{U}_t *_M \mathcal{S}_t *_M \mathcal{V}_t^{T}$. Further, we have \[
\|\mc{Y}\|_F^2=\frac{\mu_t^2}{c}\sum_{i=1}^{\min\{n_1,n_2\}}\sum_{l=1}^{n_3\cdots n_d}(\hat{\mathcal{S}}_{t,\mathcal{W},\frac{1}{\gamma\mu_t}}(i,i,l) - \hat{\mathcal{S}}_t(i,i,l) )^2,
\]
 where $\hat{\mathcal{S}}_{t,\mathcal{W},\frac{1}{\gamma\mu_t}}(i,i,l)$ and $\hat{\mathcal{S}}_t(i,i,l) $
 have the following relationship from Theorem~\ref{thm:proxsnn}:
\[
\begin{aligned}
\mathcal{S}_{t,\mathcal{W},\frac{1}{\gamma\mu_t}}(i,i,l) \times_3 M &= \hat{\mathcal{S}}_{t,\mathcal{W},\frac{1}{\gamma\mu_t}}(i,i,l) \\
&= \arg\min_{\sigma > 0} \left\{ \mathcal{W}(i,i,l) |\sigma|^p + \frac{\gamma\mu_t}{2} (\sigma - \hat{\mathcal{S}}_t(i,i,l))^2 \right\}.
\end{aligned}
\]
Now from the theory of $\ell_p$ ($0<p<1$) regularized unconstrained nonlinear programming model~\cite{Lu2014}, the local minimizer satisfies
\begin{equation}\label{eq:lemma513}
    p .\mc{W}(i,i,l).|\hat{\mathcal{S}}_{t,\mathcal{W},\frac{1}{\gamma\mu_t}}(i,i,l)|^p+ \hat{\mathcal{S}}_{t,\mathcal{W},\frac{1}{\gamma\mu_t}}(i,i,l).\gamma\mu_t.(\hat{\mathcal{S}}_{t,\mathcal{W},\frac{1}{\gamma\mu_t}}(i,i,l) - \hat{\mathcal{S}}_t(i,i,l) )=0. 
\end{equation}

 In addition, by the Theorem 2.2 of~\cite{Lu2014}, we derive that $\hat{\mathcal{S}}_{t,\mathcal{W},\frac{1}{\gamma\mu_t}}(i,i,l)$ has a lower bound. By substituting Equation~\ref{eq:lemma513} into Equation~\ref{eq:lemma511}, we have
\begin{align}
\begin{split}
\|\mathcal{Y}_{t+1(k)}\|_F 
&\leq \frac{\mu_t}{\sqrt{c}} \cdot \frac{1}{\gamma\mu_t} \sqrt{\min\{n_1,n_2\}^2 n_3\cdots n_d} \cdot \mathcal{W}(i,i,l) \cdot |\hat{\mathcal{S}}_{t,\mathcal{W},\frac{1}{\gamma\mu_t}}(i,i,l)|^{p-1} \\
&\leq \frac{1}{\sqrt{c}\gamma} \sqrt{\min\{n_1,n_2\}^2 n_3\cdots n_d} \cdot \Gamma_1^{1-p},
\end{split}
\label{eq:lemma514}
\end{align}
 where $|\mc{W}(i,i,l)|\leq 1$ and $|\hat{\mathcal{S}}_{t,\mathcal{W},\frac{1}{\gamma\mu_t}}(i,i,l)|\geq \Gamma_1 (>0)$.
This shows that $\mc{Y}_{t(k)}$ is a bounded sequence.
Also from the Algorithm \ref{alg:TWCTVTRPCA}, we have
\begin{align*}
    \|\Upsilon_{t+1}\|_F&=\mu_t\|\mc{E}_{t+1}-(\mc{M}-\mc{X}_{t+1}+\frac{\Upsilon_t}{\mu_t})\|_F\\&
    =\mu_t\|S_{\frac{\lambda}{\mu_t}\mc{W}_{\mc{E}(t)}}(\mc{M}-\mc{X}_{t+1}+\frac{\Upsilon_t}{\mu_t})-(\mc{M}-\mc{X}_{t+1}+\frac{\Upsilon_t}{\mu_t})\|_F\\&
    \leq\sum_{i=1}^{n_1}\sum_{j=1}^{n_2}\sum_{l=1}^{n_3\cdots n_d}\mu_t\frac{\lambda}{\mu_t}|\mathcal{W}_{\mc{E}(t)})(i,j,l)|\leq \lambda n_1 n_2 n_3\cdots n_d ,
    \end{align*}
 where we use the fact that $|\mathcal{W}_{\mc{E}(t)})(i,j,l)|\leq |\exp(-|(\mathcal{E}_t)(i,j,l)|/\eta)|\leq 1.$
\end{proof}

\begin{lemma}\label{eqn:lemma3}
    The sequences $\{\mc{X}_t\}_{t=1}^{\infty}$ and $\{\mc{E}_{t}\}_{t=1}^{\infty}$ generated by Algorithm \ref{alg:TWCTVTRPCA} are bounded.
\end{lemma}
\begin{proof}
 Since all subproblems for variables $\mc{X}$, $\mc{E}$, $\mc{W}_{\mc{E}}$, and $\mc{G}_k$ admit optimal solutions, then
\begin{equation}
L(\mathcal{X}_{t+1}, \mathcal{G}_{t+1(k)}, \mathcal{E}_{t+1}, \mathcal{W}_{\mc{E}(t+1)}, \mathcal{Y}_{t(k)}, {\Upsilon}_{t})\leq L(\mathcal{X}_{t}, \mathcal{G}_{t(k)}, \mathcal{E}_{t}, \mathcal{W}_{\mc{E}(t)}, \mathcal{Y}_{t(k)}, {\Upsilon}_t).
\end{equation}

\noindent Now from the formulation of augmented Lagrangian function  in Equation~\eqref{eq:trpca2}, it holds that

\begin{align*}
&L(\mathcal{X}_{t}, \mathcal{G}_{t(k)}, \mathcal{E}_{t}, \mathcal{W}_{\mc{E}(t)},\mathcal{Y}_{t(k)}, {\Upsilon}_{t})- L(\mathcal{X}_{t}, \mathcal{G}_{t(k)}, \mathcal{E}_{t}, \mathcal{W}_{\mc{E}(t)}, \mathcal{Y}_{t-1(k)}, {\Upsilon}_{t-1})\\
&= \sum_{k=1}^{\gamma}(\langle\mathcal{Y}_{t(k)} - \mathcal{Y}_{t-1(k)}, \nabla_k(\mathcal{X}_t) - \mathcal{G}_{t(k)}\rangle + \frac{\mu_t-\mu_{t-1}}{2} \|\nabla_k(\mathcal{X}_t) - \mathcal{G}_{t(k)}\|_F^2) \\
&+ \langle\Upsilon_t - \Upsilon_{t-1}, P_{\Omega}(\mathcal{M}) -\mc{E}_t- \mathcal{X}_t\rangle + \frac{\mu_t+\mu_{t-1}}{2}\|P_{\Omega}(\mathcal{M}) - \mathcal{E}_t - \mc{X}_t\|_F^2 \\
&= \frac{\mu_t + \mu_{t-1}}{2\mu_t^2}\left(\sum_{k=1}^{\gamma}\|\mathcal{Y}_{t(k)} - \mathcal{Y}_{t-1(k)}\|_F^2 + \|\Upsilon_t - \Upsilon_{t-1}\|_F^2\right).
\end{align*}

\noindent From the previous Lemma~\ref{eqn:lemma2} , we can assume that $\|\mathcal{Y}_{t(k)} - \mathcal{Y}_{t-1(k)}\|_F^2 \leq \frac{\delta_1}{2\gamma}$ and $\|\Upsilon_t - \Upsilon_{t-1}\|_F \leq \frac{\delta_1}{2}$ for all $t$, then we have

\begin{align*}
 &  L(\mathcal{X}_{t+1}, \mathcal{G}_{t+1(k)}, \mathcal{E}_{t+1}, \mathcal{W}_{\mc{E}(t+1)}, \mathcal{Y}_{t(k)}, {\Upsilon}_{t}) \\
&\leq L(\mathcal{X}_{1}, \mathcal{G}_{1(k)}, \mathcal{E}_{1}, \mathcal{W}_{\mc{E}(1)}, \mathcal{Y}_{0(k)}, {\Upsilon}_0)+ \delta_1 \sum_{t=1}^{\infty} \frac{\mu_t+\mu_{t-1}}{2(\mu_t)^2}
\\& \leq L(\mathcal{X}_{1}, \mathcal{G}_{1(k)}, \mathcal{E}_{1}, \mathcal{W}_{\mc{E}(1)}, \mathcal{Y}_{0(k)}, {\Upsilon}_0)+ \delta_1\sum_{t=1}^{\infty}\rho^{2-t}
\\& \leq L(\mathcal{X}_{1}, \mathcal{G}_{1(k)}, \mathcal{E}_{1}, \mathcal{W}_{\mc{E}(1)}, \mathcal{Y}_{0(k)}, {\Upsilon}_0)+ \delta_1\frac{\rho^2}{\rho-1}< \infty.
\end{align*}
Thereby $L(\mathcal{X}_{t+1}, \mathcal{G}_{t+1(k)}, \mathcal{E}_{t+1}, \mathcal{W}_{\mc{E}(t+1)}, \mathcal{Y}_{t(k)}, {\Upsilon}_{t}) $ being nonnegative, is bounded. From the Equation~\eqref{eq:trpca1},
\begin{align}\label{eq:conv11}
     &\frac{1}{\gamma}\sum_{k=1}^{\gamma} \|\mathcal{G}_{t(k)}\|_{\mc{W},S_p} + \lambda \mathfrak{R}(\mathcal{E}_t)\nonumber\\
     & = L(\mathcal{X}_{t}, \mathcal{G}_{t(k)}, \mathcal{E}_{t}, \mathcal{W}_{\mc{E}(t)}, \mathcal{Y}_{t-1(k)}, {\Upsilon}_{t-1}) - \frac{1}{2\mu_{t-1}}(\sum_{k\in\Gamma}\|\mathcal{Y}_t\|_F^2 + \|\Upsilon_t\|_F^2).
\end{align}
 Since $L(.)$, $\{\mc{Y}_t\}_{t=1}^{\infty}$ and $\{\Upsilon_t\}_{t=1}^{\infty}$ are all bounded, $\{\mc{G}_{t(k)}\}$ and $\{\mathfrak{R}(\mc{E}_t)\}$ are also bounded.
   We now prove that $\{\mc{E}_t\}_{t=1}^{\infty}$ is bounded by contradiction. Then let $e_t=\max_{i,j,l}|(\mc{E}_t)(i,j,l)|$. Then $\{e_t\}_{t=1}^{\infty}$ and $
\beta_t = c_{\mathcal{E}} \cdot \text{mean}(\{|(\mathcal{E}_t)_{ijl}|: (i,j,l) \in [n_1] \times [n_2] \times [n_3\cdots n_d]\})$ are also unbounded. Then using the fact $e_t\geq\frac{\beta_t}{c_{\mathcal{E}}}$, we obtain
\begin{equation}
    \mathfrak{R}(\mc{E}_{t+1})\geq \beta_t(1-\exp(-\frac{e_t}{\beta_t})\geq \beta_t(1-\exp(-\frac{1}{c_{\mathcal{E}}})). 
\end{equation}
Since $\{\beta_t\}_{t=1}^{\infty}$ is unbounded and $1-\exp(-\frac{e_t}{\beta_t})$ is a positive constant, $\{\mathfrak{R}(\mc{E}_{t+1})\}_{t=1}^{\infty}$ is also bounded and contradicts with Equation~\eqref{eq:conv11}. Therefore in the proposed Algorithm \ref{alg:TWCTVTRPCA} , $\{\mc{X}_t\}_{t=1}^{\infty}$ and $\{\mc{E}_t\}_{t=1}^{\infty}$ are bounded.
\end{proof}

\begin{lemma}[Subdifferential of weighted Schatten $p$-norm and weighted $\ell_1$ norm]\label{lemma:subgradient}
Let $\mathcal{G}_k = \mathcal{X} \times_k D_{n_k}$ be the gradient tensor along mode $k \in \Gamma$, with M-SVD $\mathcal{G}_k = \mathcal{U}_k *_M \mathcal{S}_k *_M \mathcal{V}_k^T$ under invertible transform $M$. For the weighted Schatten $p$-norm $\|\mathcal{G}_k\|_{\mathcal{W},S_p}$ ($0 < p < 1$), the subdifferential is:  
\[
\partial \|\mathcal{G}_k\|_{\mathcal{W},S_p} = \frac{1}{\sqrt{c}} \mathcal{U}_k *_M \left(\partial f(\hat{\mathcal{S}}_k) \times_3 M^{-1}\right) *_M \mathcal{V}_k^T,
\]  
where $\partial f(\hat{\mathcal{S}}_k)$ is an f-diagonal tensor with entries:  
\[
\partial f\left(\hat{\mathcal{S}}_k(i,i,l)\right) = {\mathcal{W}(i,i,l) \cdot p \cdot (\hat{\mathcal{S}}_k(i,i,l)+\epsilon)^{p-1}},  
\]  
where $\epsilon > 0$ (a small constant to avoid division by zero).
For the norm \(\|\mathcal{E}\|_{\mathcal{W}_{\mathcal{E},1}} = \sum_{i,j,l} \mathcal{W}_{\mathcal{E}}(i,j,l) |\mathcal{E}(i,j,l)|\), the subdifferential is:
\[
\partial \|\mathcal{E}\|_{\mathcal{W}_{\mathcal{E},1}} = \left\{ \mathcal{H}_1 \in \mathbb{R}^{n_1 \times n_2 \times \cdots \times n_3\cdots n_d} \;\Big|\; \mathcal{H}_1(i,j,l) = \begin{cases} 
\mathcal{W}_{\mathcal{E}}(i,j,l) \cdot \mathrm{sign}(\mathcal{E}(i,j,l)), & \mathcal{E}(i,j,l) \neq 0, \\
\left[-\mathcal{W}_{\mathcal{E}}(i,j,l), \mathcal{W}_{\mathcal{E}}(i,j,l)\right], & \mathcal{E}(i,j,l) = 0.
\end{cases} \right\}.
\]

\end{lemma}

\begin{proof}

   Apply transform to $\mathcal{G}_k$, yielding $\hat{\mathcal{G}}_k = \mathcal{G}_k \times_3 M$, which block-diagonalizes as:  
   \[
   \operatorname{bdiag}(\hat{\mathcal{G}}_k) = \operatorname{bdiag}(\hat{\mathcal{U}}_k) \operatorname{bdiag}(\hat{\mathcal{S}}_k) \operatorname{bdiag}(\hat{\mathcal{V}}_k)^T,
   \]  
   where $\hat{\mathcal{S}}_k(i,i,l) = \hat{\sigma}_i^{(l)}(\mathcal{G}_k)$ are the singular values.  Thus, $||\mathcal{G}_k\|_{\mathcal{W},S_p}$ can be viewed as a singular value function of the matrix $\operatorname{bdiag}(\hat{\mathcal{G}}_k)$, i.e $||\mathcal{G}_k\|_{\mathcal{W},S_p}=\frac{1}{\sqrt{c}}\|\operatorname{bdiag}(\hat{\mathcal{G}}_k)\|_{\mathcal{W},S_p}$
   To avoid singularity at $\hat{\mathcal{S}}_k(i,i,l) = 0$, define:  
   \[
  {p\cdot\mathcal{W}(i,i,l)(\hat{\mathcal{S}}_k(i,i,l)+\epsilon)^{p-1}}.
   \]  
   This ensures boundedness:  
   \[
   \left|\partial f\left(\hat{\mathcal{S}}_k(i,i,l)\right)\right| \leq \frac{p\cdot\mathcal{W}(i,i,l) }{\epsilon^{1-p}} \quad \forall \hat{\mathcal{S}}_k(i,i,l) \geq 0.
   \]  
    We denote the subdifferential in the transform domain by \(\overline{\partial}\).  
Then the subdifferential of $\|{\operatorname{bdiag}}(\hat{\mathcal{G}}_k)\|_{\mathcal{W},S_p}$ in the transform domain is given by:
\[
\overline{\partial} \|{\operatorname{bdiag}}(\hat{\mathcal{G}}_k)\|_{\mathcal{W},S_p} =\frac{1}{\sqrt{c}} {\operatorname{bdiag}}(\hat{\mathcal{U}}_k) {\operatorname{bdiag}}(\overline{\partial} f(\hat{\mathcal{S}}_k)) {\operatorname{bdiag}}(\hat{\mathcal{V}}_k)^T,
\]
  Then, we conclude the expression for the subdifferential of $\|\mathcal{G}_k\|_{\mathcal{W},S_p}$ in the original domain, that is,

   \[
   \partial \|\mathcal{G}_k\|_{\mathcal{W},S_p} = \frac{1}{\sqrt{c}} \mathcal{U}_k * \left(\partial f(\hat{\mathcal{S}}_k) \times_3 M^{-1}\right) * \mathcal{V}_k^T.\]

\noindent For the entry-wise weighted \(\ell_1\) norm, the subdifferential at \(\mathcal{E}\) is defined entry-wise. For each \((i,j,l)\):
\begin{itemize}
    \item If \(\mathcal{E}(i,j,l) \neq 0\), the subgradient is unique and equals to \( \mathcal{W}(i,j,l) \cdot \mathrm{sign}(\mathcal{E}(i,j,l))\).
    \item If \(\mathcal{E}(i,j,l) = 0\), the subgradient lies in the interval \(\left[-\mathcal{W}(i,j,l), \mathcal{W}(i,j,l)\right]\).
\end{itemize}
Aggregating over all entries, the subdifferential is the set of all tensors \(\mathcal{H}_1\) satisfying the above conditions.
\end{proof}

The iterates generated by Algorithm~\ref{alg:TWCTVTRPCA} are asymptotically regular, meaning the norm of the difference between consecutive iterates converges to zero, indicating stabilization of the sequence. The next theorem establishes the convergence of Algorithm~\ref{alg:TWCTVTRPCA}.
\begin{theorem}\label{thm:stablization}
Let the weight tensor $\mc{W}$ have non-descending diagonal elements in its frontal slices and the sparse weight tensor $\mc{W}_{\mc{E}}$ be bounded. Under these conditions, then $\{\mc{X}_t\}_{t=1}^{\infty}$, $\{\mathcal{G}_{k({t})}\}_{t=1}^{\infty}$ for $k\in \Gamma$, and $\{\mc{E}_{t}\}_{t=1}^{\infty}$ in Algorithm 2 satisfy
\begin{enumerate}
    \item  $\lim_{t \to \infty} \|P_{\Omega}(\mc{M}) - \mc{X}_t - \mc{E}_t\|_F = 0$;
    \item $\lim_{t \to \infty} \|\mathcal{X}_{t+1} - \mathcal{X}_t\|_F  = \lim_{t \to \infty} \|\mc{E}_{t+1} - \mc{E}_{t}\|_F = 0$;
    \item  $\lim_{t \to \infty} \|\mathcal{G}_{t+1(k)} - \mathcal{G}_{t(k)}\|_F = 0$. 
\end{enumerate}
Furthermore, the sequence $\{\mathcal{X}_{t}, \mathcal{G}_{t(k)}, \mathcal{E}_{t}, \mathcal{W}_{\mc{E}(t)},\mathcal{Y}_{t(k)}, {\Upsilon}_{t}\}$ is a Cauchy sequence and their limit point $\{\mathcal{X}^*, \mathcal{G}_{k}^*, \mathcal{E}^*, \mathcal{W}_{\mc{E}}^*,\mathcal{Y}_{k}^*, {\Upsilon}^*\}$ satisfies the Karush-Kuhn-Tucker (KKT) conditions of the Algorithm~\ref{alg:TWCTVTRPCA}.
\end{theorem}
\begin{proof}
   To prove the above result, we use the Bolzano-Weierstrass theorem, which states that the bounded sequence in Lemmas \ref{eqn:lemma2}, \ref{eqn:lemma3} has at least one convergent subsequence. There exist at least one limit point for $\{\mc{X}_t\}_{t=1}^{\infty}$,$\{\mc{E}_{t}\}_{t=1}^{\infty}$ and $\{\mathcal{G}_{t({k})}\}_{t=1}^{\infty}$ . Since \( \mu_t \to 0 \) as \( t \to \infty \), we deduce that,
 \begin{equation}
\lim_{t \to \infty}\|P_{\Omega}(\mc{M}) - \mc{X}_t - \mc{E}_t\|_F = \lim_{t \to \infty} \frac{1}{\mu_{t-1}}\|P_{\Omega}(\Upsilon_t) - P_{\Omega}(\Upsilon_{t-1})\|_F =P_{\Omega}(\lim_{t \to \infty} \frac{1}{\mu_{t-1}}\|\Upsilon_t - \Upsilon_{t-1}\|_F)= 0.
\end{equation}

\noindent Therefore, the limit point of such a sequence is also feasible for the objective function. Suppose $\mc{T}_t=\nabla_k(\mathcal{X}_{t+1}) + \frac{\mathcal{Y}_{t(k)}}{\mu_t}$, then 
\begin{equation} \label{eq:thm543}
\begin{aligned}
    \|\mathcal{G}_{t+1(k)}-\mathcal{G}_{t(k)}\|_F
    &= \|\mathfrak{D}_{\mathcal{W},\frac{1}{\gamma\mu_t}}(\mc{X}_t)-\mc{X}_t\|_F + \|\mc{X}_t-\mathcal{G}_{t(k)}\|_F \\
    &= \frac{1}{\sqrt{c}}\left\|\operatorname{bdiag}\left(\mathfrak{D}_{\mathcal{W},\frac{1}{\gamma\mu_t}}(\hat{\mc{X}}_t)\right) - \operatorname{bdiag}(\hat{\mc{X}}_t)\right\|_F + \left\|\frac{\mathcal{Y}_t}{\mu_t} + \frac{\mathcal{Y}_t-\mathcal{Y}_{t-1}}{\mu_{t-1}}\right\|_F \\
    &\leq \frac{\|\mathcal{W}\|_F}{\sqrt{c}\gamma\mu_t} + \left\|\frac{\mathcal{Y}_t}{\mu_t} + \frac{\mathcal{Y}_t-\mathcal{Y}_{t-1}}{\mu_{t-1}}\right\|_F.
\end{aligned}
\end{equation}
Since $\lim_{t \to \infty}\frac{1}{\mu_t}=0$, and the sequences $\{\mc{Y}_t\}_{t=1}^{\infty}$ is bounded with $\|\mc{W}\|_F$ being bounded, we have 
\begin{equation}
\lim_{m,n \to \infty}\|\mathcal{G}_{m(k)} - \mathcal{G}_{n(k)}\|_F \leq \lim_{m,n \to \infty}\sum_{t=n}^{m-1} \|\mathcal{G}_{t+1(k)} - \mathcal{G}_{t(k)}\|_F=0 \quad \text{for any } m > n.
\end{equation}

Since for each \( k \in \Gamma \), the sequence \( \{\mathcal{G}_{k(t)}\}_{t=1}^\infty \) is Cauchy in the Frobenius norm, 
and noting that \( \mathcal{G}_k = \mathcal{X} \times_k D_{n_k} \) where $D_{n_k}$ is a fixed circulant matrix, 
it follows that the sequence \( \{\mathcal{X}_t\}_{t=1}^\infty \) is also Cauchy. 
This is because multiplication by a bounded linear operator (in this case, \( \times_k D_{n_k} \)) preserves the Cauchy property.
In addition, from Equation~\eqref{eq:thm543}, we deduce that
\begin{equation}
    \label{eq:thm542}
    \lim_{t\to\infty} \|\mathcal{X}_{t+1} - \mathcal{X}_t\|_F = 0.
\end{equation}

\noindent Furthermore, for the sequence \(\{\mathcal{E}_{t}\}_{t=1}^{\infty}\), we have
 \begin{equation}\label{eq:thm544}
\begin{aligned}
    \|\mathcal{E}_{t+1} - \mathcal{E}_{t}\|_F 
    &= \left\| S_{\frac{\lambda}{\mu_t}\mathcal{W}_{\mathcal{E}(t)}}\left( \mathcal{M} - \mathcal{X}_{t+1} + \frac{\Upsilon_t}{\mu_t} \right) - \left( \mathcal{M} - \mathcal{X}_{t+1} + \frac{\Upsilon_t}{\mu_t} \right) \right\|_F \\
    &\quad + \|\mathcal{X}_t - \mathcal{X}_{t+1}\|_F 
    + \left\| \frac{\Upsilon_t}{\mu_t} + \frac{\Upsilon_t - \Upsilon_{t-1}}{\mu_{t-1}} \right\|_F \\
    &\leq \frac{\lambda}{\mu_t} n_1 n_2 n_3\cdots n_d 
    + \|\mathcal{X}_t - \mathcal{X}_{t+1}\|_F 
    + \frac{\|\Upsilon_t\|_F}{\mu_t} 
    + \frac{\|\Upsilon_t - \Upsilon_{t-1}\|_F}{\mu_{t-1}}.
\end{aligned}
\end{equation}
Combining Equation~\eqref{eq:thm542} with the fact that \(\lim_{t \to \infty} \frac{1}{\mu_t} = 0\), we deduce that
\[
\lim_{t \to \infty} \|\mathcal{E}_{t+1} - \mathcal{E}_{t}\|_F = 0.
\]
Moreover, we have
\begin{equation}
\lim_{m,n \to \infty} \|\mathcal{E}_{m} - \mathcal{E}_{n}\|_F 
\leq \lim_{m,n \to \infty} \sum_{t=n}^{m-1} \|\mathcal{E}_{t+1} - \mathcal{E}_{t}\|_F = 0 
\quad \text{for any } m > n.
\end{equation}
This concludes that the sequence \(\{\mathcal{E}_t\}_{t=1}^\infty\) is Cauchy.

The following KKT conditions establish the first-order optimality for the limit points generated by Algorithm~\ref{alg:TWCTVTRPCA}.
Now, for $0< p < 1$, the problem derived from Equation~\eqref{eq:RTCmodel} becomes a non-convex problem, and KKT conditions become necessary but not sufficient for optimality. Since the limit points of the sequences are \(\{\mathcal{X}^*, \mathcal{G}_{k}^*, \mathcal{E}^*, \mathcal{W}_{\mathcal{E}}^*, \mathcal{Y}_{k}^*, \Upsilon^*\}\), 
the first-order optimality conditions corresponding to the subproblems~\eqref{eq:trpca41} and~\eqref{eq:trpca5} lead to the following system:
\begin{equation}\label{eq:kkt}
\begin{aligned}
& 0 \in \partial\left(\frac{1}{\gamma} \sum_{k=1}^\gamma \|\mathcal{G}_k^*\|_{\mathcal{W}, S_p} \right) - \mu\left( \nabla_k(\mathcal{X}^*) - \mathcal{G}_k^* + \frac{\mathcal{Y}_k^*}{\mu} \right), \\
& 0 \in \partial \left( \|\mathcal{E}^*\|_{\mathcal{W}_{\mathcal{E}}, 1} \right) + \frac{\mu}{\lambda}\left( P_{\Omega}(\mathcal{M}) - \mathcal{X}^* - \mathcal{E}^* + \frac{\Upsilon^*}{\mu} \right),
\end{aligned}
\end{equation}
Here, $\partial$ represents the subdifferential operator corresponding to the associated nonconvex regularization terms. By reorganizing terms, we can further simplify the optimality conditions as:
\begin{equation}\label{eq:kkt2}
\left.
\begin{aligned}
& 0 \in \frac{1}{\gamma} \partial \left( \sum_{k=1}^\gamma \|\mathcal{G}_k^*\|_{\mathcal{W}, S_p} \right) + \mathcal{Y}_k^*, \\
& 0 \in \lambda \partial \left( \|\mathcal{E}^*\|_{\mathcal{W}_{\mathcal{E}}, 1} \right) + \Upsilon^*,
\end{aligned}  \right\}
\quad \text{subject to} \quad \mathcal{G}_k^* = \nabla_k(\mathcal{X}^*), \quad P_{\Omega}(\mathcal{M}) = \mathcal{X}^* + \mathcal{E}^*.
\end{equation}

\noindent Therefore, any limit point satisfies the Karush-Kuhn-Tucker (KKT) conditions associated with problem~\eqref{eq:trpca2}.

\end{proof}

\section{Numerical Experiment and Discussions}
In this section, we analyze the performance of Algorithm~\ref{alg:TWCTVTRPCA} through empirical convergence behavior and phase transition diagrams demonstrating recovery guarantees. Also, we evaluate the effectiveness of our proposed methods against state-of-the-art techniques using various datasets, including color images, multiple spectral images, hyperspectral images, and color videos. The experiments are conducted in MATLAB R2023b running on Windows 11, using an Intel(R) Core(TM) $i9$-$12900$K ($3.2$~GHz) processor with 32~GB RAM.  All experiments use the same algorithmic parameters: maximum iterations $T_{\max} = 500$, tolerance $\epsilon = 10^{-8}$, initial penalty parameter $\mu_0 = 10^{-6}$, and penalty update ratio $\rho = 1.1$. The regularization parameter is set as $\lambda = 1/\sqrt{n_1 n_2 n_3..n_d / \min(n_1, n_2)}$ following standard practice.
\subsection{Choice of transform matrix on synthetic data}
In this experiment, we test the proposed TWCTV framework on fourth-order ($d=4$) synthetic low-rank tensors $\mathcal{M} \in \mathbb{R}^{n \times n \times n_3 \times n_4}$ with $n_3 = n_4 = 20$ under different invertible linear transforms matrix $M$. For each configuration $(n, r)$, 
the ground-truth tensor is generated via the higher-order $M$-product as
\begin{equation}\label{eq:generatetensor}
    \mathcal{M} = \mathcal{M}_1 *_{M} \mathcal{M}_2,
\end{equation}
where $\mathcal{M}_1 \in \mathbb{R}^{n \times r \times n_3 \times n_4}$ 
and $\mathcal{M}_2 \in \mathbb{R}^{r \times n \times n_3 \times n_4}$ have 
i.i.d.\ entries drawn from $\mathcal{N}(0, 1/n)$, so that 
$\mathrm{rank}_{\mathrm{M-SVD}}(\mathcal{M}) = r$ by construction. 
We choose the invertible linear transform $M \in \mathbb{R}^{n_k \times n_k}$ applied along modes $k = 3, 4$, including the Discrete Fourier Transform (DFT), the Discrete Cosine Transform (DCT), the Discrete Wavelet Transform (DWT, Haar basis), and a Random Orthogonal Transform (ROT). A random observation mask 
$\Omega$ is generated by uniformly sampling $m = \lfloor p \cdot 
n^2 n_3 n_4 \rfloor$ entries without replacement, with a fixed 
sampling ratio $p = 0.5$, and the observed tensor is 
$\mathcal{P}_{\Omega}(\mathcal{M})$. Recovery quality is measured 
by the relative error
\begin{equation*}
    \mathrm{RE} = \frac{\|\hat{\mathcal{X}} - \mathcal{M}\|_F}
                       {\|\mathcal{M}\|_F},
\end{equation*}
where $\hat{\mathcal{X}}$ denotes the recovered tensor.
 The recovered relative errors and runtimes are reported 
in Table~\ref{tab:tc4}.

\begin{table}[!h]
\centering

\caption{Recovered results on fourth-order synthetic tensor
($d = 4$, $n_3 = n_4 = 20$, $p = 0.5$).}
\label{tab:tc4}
\renewcommand{\arraystretch}{1.1}
\setlength{\tabcolsep}{6pt}

\begin{tabular}{cc|lr|lr|lr|lr}
\hline
\multirow{2}{*}{$n$} & \multirow{2}{*}{$r$}
& \multicolumn{2}{c|}{DCT}
& \multicolumn{2}{c|}{DFT}
& \multicolumn{2}{c|}{DWT}
& \multicolumn{2}{c}{ROT} \\
& & RE & Time
  & RE & Time
  & RE & Time
  & RE & Time \\
\hline

\multirow{2}{*}{30}
& 3 & $1.78\times10^{-8}$ & 7.56
    & $7.43\times10^{-8}$ & 25.69
    & $8.41\times10^{-6}$ & 29.51
    & $8.45\times10^{-6}$ & 29.39 \\
& 5 & $3.97\times10^{-8}$ & 7.53
    & $7.84\times10^{-8}$ & 25.75
    & $8.46\times10^{-6}$ & 29.52
    & $8.47\times10^{-6}$ & 29.23 \\
\hline

\multirow{4}{*}{50}
& 3  & $1.09\times10^{-7}$ & 18.71
      & $6.83\times10^{-7}$ & 27.36
      & $8.37\times10^{-6}$ & 31.78
      & $8.40\times10^{-6}$ & 31.96 \\
& 5  & $1.74\times10^{-7}$ & 19.01
      & $7.43\times10^{-7}$ & 27.39
      & $8.42\times10^{-6}$ & 32.16
      & $8.44\times10^{-6}$ & 31.86 \\
& 8  & $3.79\times10^{-7}$ & 19.12
      & $7.81\times10^{-7}$ & 28.80
      & $3.22\times10^{-5}$ & 35.01
      & $8.48\times10^{-6}$ & 34.97 \\
& 10 & $4.65\times10^{-7}$ & 19.99
      & $7.92\times10^{-7}$ & 30.06
      & $5.46\times10^{-6}$ & 34.64
      & $7.55\times10^{-5}$ & 34.46 \\
\hline
\end{tabular}

\end{table}
As shown in Table~\ref{tab:tc4}, the DCT consistently produces the lowest relative error across all tested cases, often outperforming DFT, DWT, and ROT by one to two orders of magnitude. This improvement is mainly due to the energy compaction property of the DCT: for real-valued tensors with smooth variations, most of the signal energy is concentrated in a few low-frequency components, making the data more low-rank in the transform domain and easier to recover using norm minimization. In contrast, the DFT assumes periodic boundary conditions, which can introduce artificial discontinuities at the boundaries of non-periodic data and spread the energy across many frequencies, thereby weakening the low-rank structure. The DWT and ROT also result in relatively high errors in all settings, indicating that neither the wavelet basis nor a random orthogonal basis matches the underlying low-rank structure of the synthetic tensors well.

It is also worth noting that there is no universally optimal choice of the transform matrix $M$ for all datasets. In practice, the selection depends on the structural characteristics of the data and the prior assumptions one intends to impose.


\subsection{Relative error convergence}
In this study, we investigate how the algorithm converges when using different transform matrices (DCT and DFT) within the $M$-product framework. We assess the Algorithm~\ref{alg:TWCTVTRPCA}'s convergence by tracking the relative errors between successive iterations during the reconstruction process.  Specifically, the relative errors for \(\mc{X}\) and \(\mc{E}\) at the \((t+1)\)-th iteration are defined as follows:
\begin{equation*}
    \text{Relative error of $\mc{X}$} = \min\left\{1, \frac{\|\mc{X}_{t+1} - \mc{X}_t\|_F}{\|\mc{X}_t\|_F}\right\}, \quad
\text{Relative error of $\mc{E}$} = \min\left\{1, \frac{\|\mc{E}_{t+1} - \mc{E}_t\|_F}{\|\mc{E}_t\|_F}\right\}.
\end{equation*}

\begin{figure}[!h]
\centering
\begin{tabular}{cc}
\includegraphics[width=0.5\textwidth]{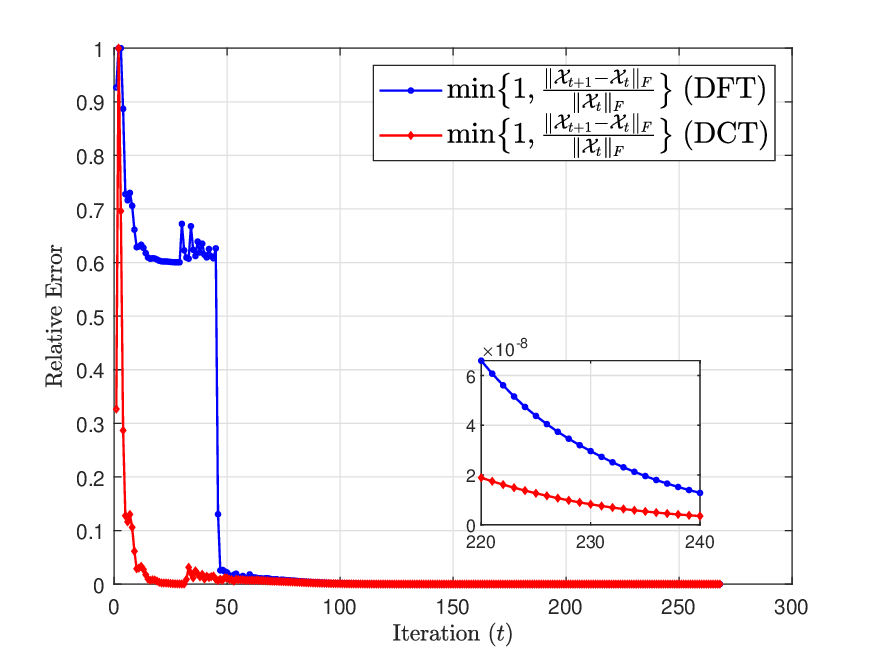} &
\includegraphics[width=0.5\textwidth]{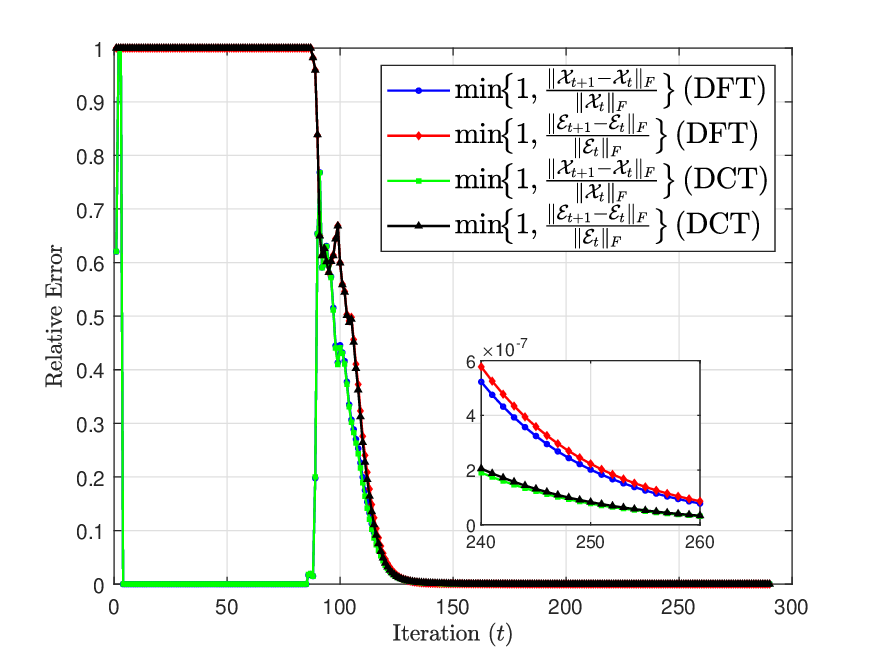} \\
(a) Tensor completion. & (b) TRPCA.
\end{tabular}
\caption{Relative error curves showing convergence behavior of: (a) Tensor completion (on the tensor of size $256\times 256\times 3$) with TWCTV regularization and (b) TRPCA (on the tensor of size $256\times 256\times 3$) with  TWCTV regularization and weighted $\ell_1$ sparse term.}
\label{fig:relerror}
\end{figure}
Figure~\ref{fig:relerror} shows that Algorithm~\ref{alg:TWCTVTRPCA} converges after certain iterations and achieve steady convergence. The convergence curve also demonstrates that using DCT yields faster results than DFT, as DCT relies on real arithmetic, while DFT involves complex arithmetic. In addition, the stopping criterion ensures that the algorithm halts when the relative changes fall below a predefined threshold, further validating its stability.

\subsection{Phase transition}
In this subsection, we evaluate Algorithm~\ref{alg:TWCTVTRPCA}'s recovery guarantee under various sampling rates and rank conditions. A synthetic low-tubal-rank tensor $\mathcal{X}$ is constructed via the $M$-product of two tensors as defined in~\eqref{eq:generatetensor}, guaranteeing that $\mathcal{X}$ has tubal rank $r$. The recovery experiments are performed on a $40 \times 40 \times 20$ tensor with varying tubal rank $r$ and sampling rate (ratio of observed entries $m$ to total entries $n_1n_2n_3$). Each configuration is tested 10 times using both weighted and non-weighted TCTV approaches. A recovery is deemed successful when the relative Frobenius norm error satisfies:
\begin{equation*}
\frac{\|\mathcal{X}-\bar{\mathcal{X}}\|_F}{\|\mathcal{X}\|_F} < 10^{-3},
\end{equation*}
where $\bar{\mathcal{X}}$ denotes the recovered tensor.

\begin{figure}[!ht]
\centering
\begin{tabular}{cc}
\includegraphics[width=0.5\textwidth]{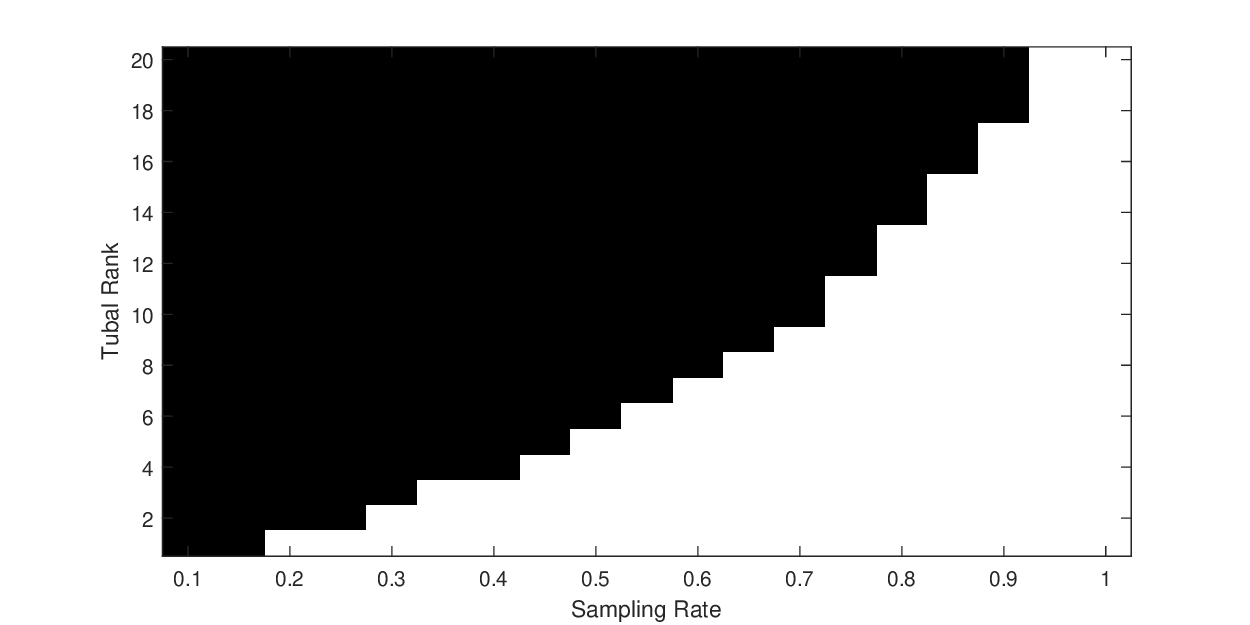} &
\includegraphics[width=0.5\textwidth]{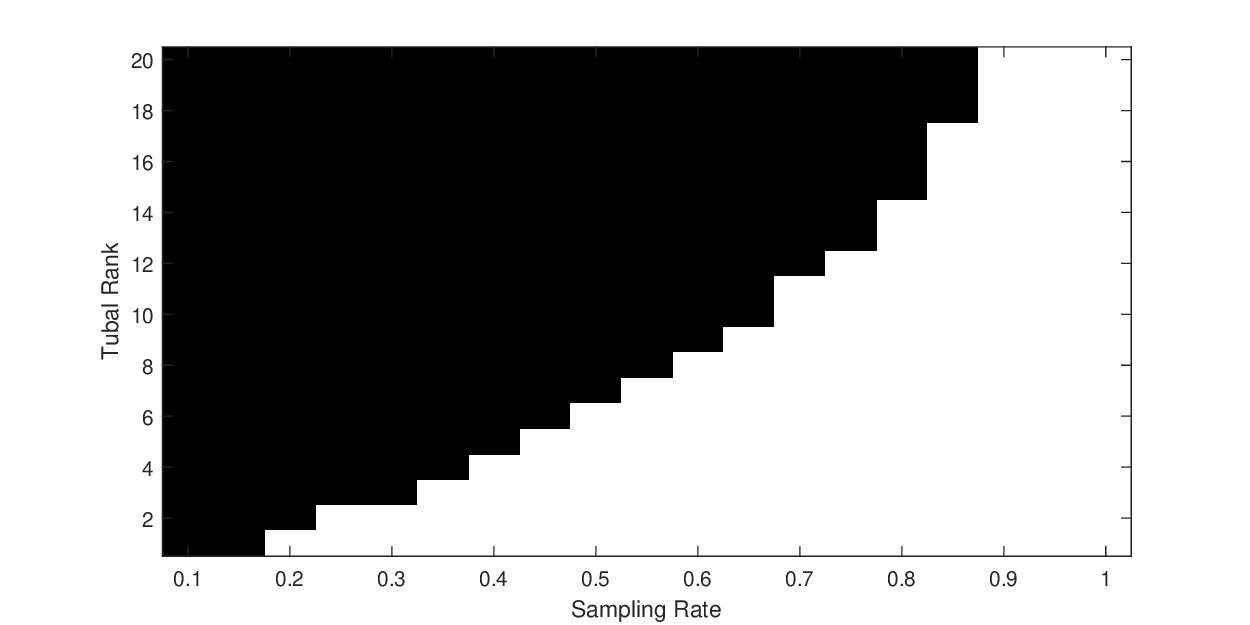} \\
(a)Tensor completion with TCTV \cite{Wang202310990} regularizer. & (b) Tensor completion with TWCTV regularizer.
\end{tabular}
\caption{Phase transition diagrams for tensor completion (size: $40 \times 40 \times 20$): (a) TCTV regularizer with successful cases (white): 142 (37.37\%) and (b) Weighted Schatten $p$-norm ($p=0.9$) based TWCTV regularizer with successful cases (white): 160 (42.11\%).}
\label{fig:phase_transition}
\end{figure}

The phase transition diagrams in Figure~\ref{fig:phase_transition} visualize the success rates of recovery, where each cell represents the ratio of successful recovery to total trials (10). Success rates range from 0 to 1, with white cells indicating perfect recovery (10/10 successful trials) and black cells indicating failure. The results demonstrate the superior performance of our weighted approach compared to the non-weighted version, achieving more successful cases than the TCTV.

\subsection{$p$-sensitivity analysis and ablation study}
\subsubsection{$p$-sensitivity analysis}
In this subsection we will conduct a comprehensive sensitivity analysis of Schatten-$p$ on synthetic data. The synthetic tensor is generated with dimensions $80 \times 80 \times 20$ using a DCT-based M-product framework with tubal rank $r=5$. We introduce $8\%$ sampling rate with $2\%$ corruption ratio and outlier magnitude of $0.4$. The Schatten-$p$ exponent is varied from $p \in \{0.1, 0.2, \ldots, 0.9\}$. For each $p$ value, we perform 10 independent trials with different random seeds to ensure statistical reliability. Figure~\ref{fig:p_sensitivity} illustrates the impact of the Schatten-$p$ exponent on tensor recovery performance across three evaluation metrics. 
As shown in Fig.~\ref{fig:p_sensitivity}(a), smaller $p$ values yield consistently lower solver residuals 
$\|P_{\Omega}(\mathcal{M})-\mathcal{X}-\mathcal{E}\|_F$, demonstrating that nonconvex regularization more effectively enforces low-rank structure 
by suppressing minor singular values. 
The best recovery accuracy is achieved within $p\!\in[0.6,0.9]$
confirming the benefit of nonconvexity for challenging recovery tasks. 
Figure~\ref{fig:p_sensitivity}(b) shows that computational time remains nearly constant across all $p$ around $0.6-0.9$. 
Finally, Fig.~\ref{fig:p_sensitivity}(c) reports the number of iterations required for convergence.
This efficiency arises because stronger nonconvexity gives more selective shrinkage on smaller singular values. 
Overall, moderate nonconvexity ($p\!\approx\!0.8$ or $0.9$) achieves the optimal balance among reconstruction accuracy, stability, 
and convergence speed, making it a robust default choice for tensor recovery problems.

\begin{figure}[!h]
\centering
\begin{tabular}{ccc}
\includegraphics[width=0.32\textwidth]{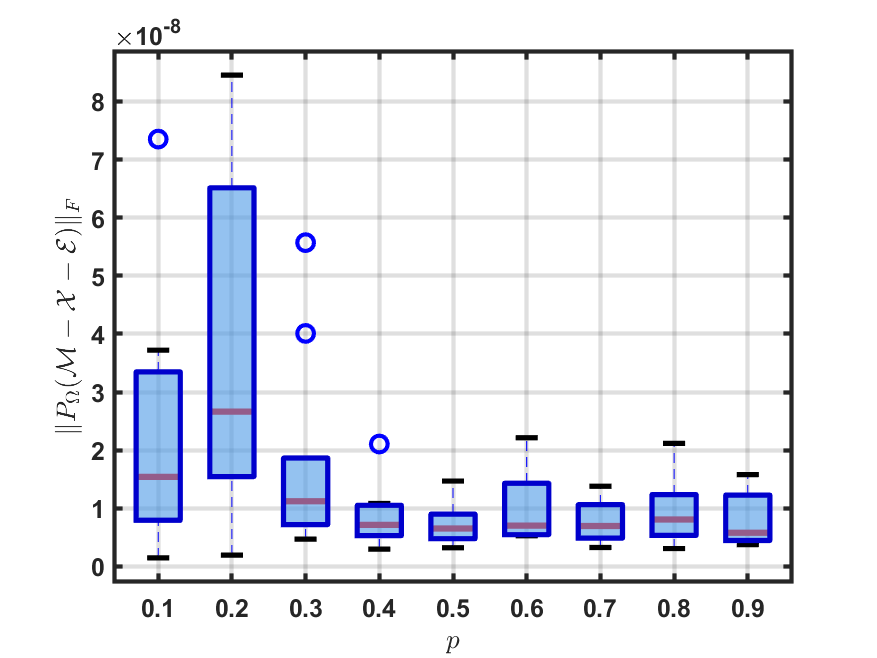} &
\includegraphics[width=0.32\textwidth]{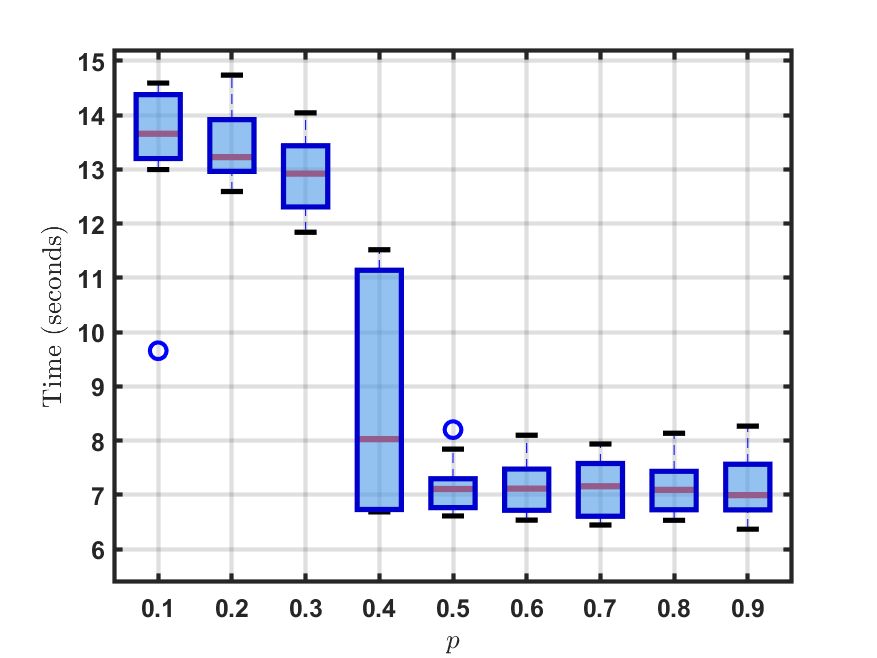} &
\includegraphics[width=0.32\textwidth]{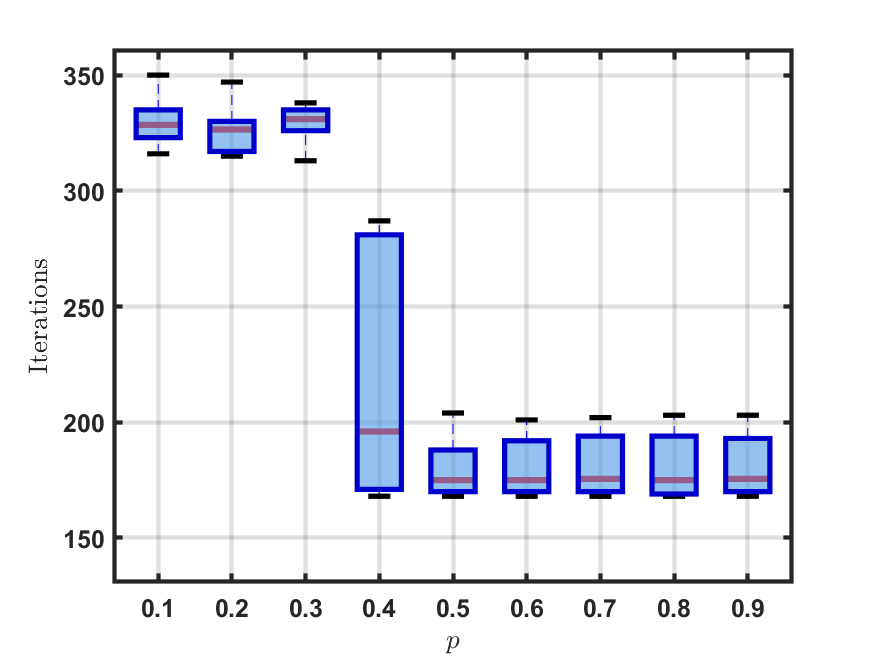} \\
(a) Solver residual. & (b) Computation time. & (c) Iterations to convergence.
\end{tabular}
\caption{Sensitivity analysis of the Schatten-$p$ exponent on synthetic tensor data ($80\times80\times20$, sampling rate $8\%$, corruption ratio $2\%$).}
\label{fig:p_sensitivity}
\end{figure}

\subsubsection{Ablation study: two nonconvex vs. one nonconvex regularization}

In this subsection we demonstrate the superiority of using two nonconvex regularizers. We compare the proposed model (TWCTV + weighted $\ell_1$) against a baseline using only one nonconvex term (TWCTV + standard convex $\ell_1$). 
Experiments are conducted on synthetic tensors of sizes $80\times80\times20$ and $100\times100\times20$ with $8\%$ sampling rate and $2\%$ corruption ratio, using identical algorithmic parameters for both models. As shown in Table~\ref{tab:ablation_compact}, the two-nonconvex formulation consistently delivers superior convergence accuracy across tensor sizes. 
For example, the suggested model obtains solver residual reductions of $79.3\%$ and $65.4\%$ on the $80\times80\times20$ and $100\times100\times20$ tensors, respectively, demonstrating the efficacy of adaptive weighting in improving sparse–low-rank separation.The execution timings for both tensor sizes are virtually equal (about $5.0$s for $80\times80\times20$ and $8.2$s for $100\times100\times20$), with iteration counts differing by no more than $3\%$. 
This indicates that adaptive weight updates in the weighted $\ell_1$ phrase have little overhead relative to tensor decomposition procedures.

\begin{table}[!h]
\centering
\caption{Ablation study comparing two-nonconvex and one-nonconvex models on synthetic tensors. Solver residuals measured as $\|P_{\Omega}(\mathcal{M})-\mathcal{X}-\mathcal{E}\|_F$.}
\label{tab:ablation_compact}
\begin{tabular}{lcccc}
\hline
\textbf{Tensor Size} & \textbf{Model} & \textbf{Solver Residual} & \textbf{Time (s)} & \textbf{Iterations} \\
\hline
\multirow{2}{*}{$80\times80\times20$} 
& Two Nonconvex & $2.82\times10^{-9}$ & $5.06$ & $170$ \\
& One Nonconvex & $1.36\times10^{-8}$ & $4.99$ & $164$ \\
\cline{2-5}

\hline
\multirow{2}{*}{$100\times100\times20$} 
& Two Nonconvex & $6.65\times10^{-9}$ & $8.22$ & $175$ \\
& One Nonconvex & $1.92\times10^{-8}$ & $8.18$ & $173$ \\
\cline{2-5}

\hline
\end{tabular}
\end{table}

\subsection{Application on image impainting}

In this subsection, we demonstrate tensor completion's effectiveness for image restoration tasks, leveraging the natural multi-dimensional structure of color images and multispectral image (MSI) data where each color channel or spectral band forms a mode of the tensor. We compare the proposed TWCTV method with current state-of-the-art tensor recovery techniques to validate its effectiveness. Various tensor completion methods, including SNN \cite{liu2012snn}, BCPF \cite{zhao2015bcpf}, KBR \cite{Qi2018kbr}, IRTNN \cite{wang2022irtnn}, TNN \cite{wenjin2022tnn}, TNN+TV \cite{qiu2021tnntv}, SNN+TV \cite{li2017snntv}, SPC+TV \cite{yokota2016spctv}, TCTV \cite{Wang202310990}, and the proposed TWCTV based tensor completion method, are applied to recover the missing entries. In our experiments, we set $p=0.9$ and choose $M$ as the DCT matrix for the proposed method. We evaluate reconstruction quality using three metrics: peak signal-to-noise ratio (PSNR), which assesses pixel-wise accuracy; structural similarity index (SSIM), which measures structural similarity; and feature similarity index (FSIM), which evaluates feature preservation. 

\subsubsection{Color image inpainting} In this part, we apply the proposed algorithm to color image recovery and compare it with other state-of-the-art methods. We use the image ``Facade''\footnote{\url{https://sipi.usc.edu/database/}} as the input, which is a tensor of size $256 \times 256 \times 3$. In our experiments, we evaluate performance under both random and text masking. For the random masking, we remove pixels at uniform rates of 
$10\%, 20\%, 50\%, \text{and}\ 80\%$, creating observations with different levels of missing data. For the text masking, we overlay text at predefined positions to mimic real-world cases such as watermarks. The proposed TWCTV method consistently outperforms existing methods, including SNN, TNN, BCPF, KBR, and TCTV, in terms of reconstruction quality, achieving higher metric values across different sampling rates and the text mask. The results shown in Figure~\ref{fig:comparison} and Table~\ref{tab:imp1} demonstrate that the proposed method effectively preserves both the colors and textures.

\begin{figure}[!ht]
\centering

\begin{tabular}{cccccc}
\includegraphics[width=0.15\textwidth]{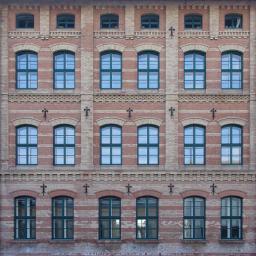} &
\includegraphics[width=0.15\textwidth]{image/Color_TC_Observed.jpg} &
\includegraphics[width=0.15\textwidth]{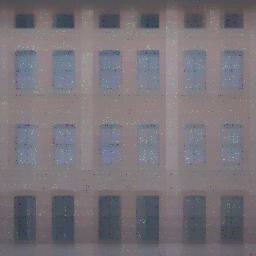} &
\includegraphics[width=0.15\textwidth]{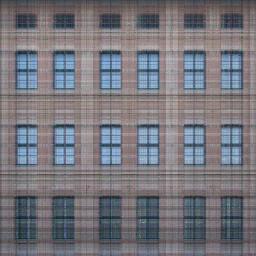} &
\includegraphics[width=0.15\textwidth]{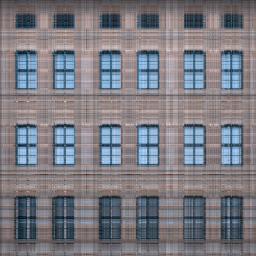} &
\includegraphics[width=0.15\textwidth]{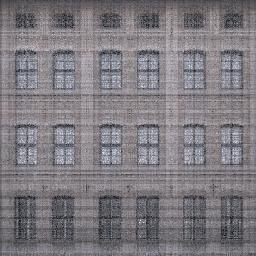} \\
Original & Observed & SNN~\cite{liu2012snn} & SNN+TV~\cite{li2017snntv} & BCPF~\cite{zhao2015bcpf} & KBR~\cite{Qi2018kbr} \\
\includegraphics[width=0.15\textwidth]{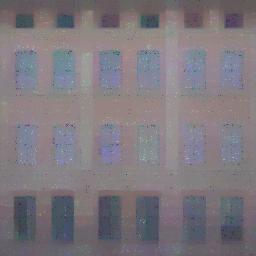} &
\includegraphics[width=0.15\textwidth]{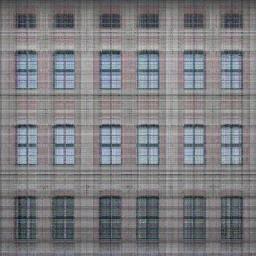} &
\includegraphics[width=0.15\textwidth]{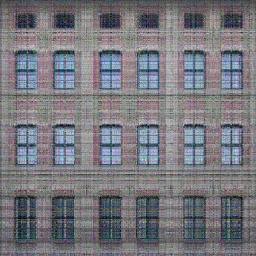} &
\includegraphics[width=0.15\textwidth]{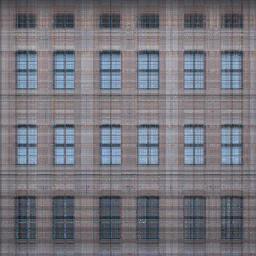} &
\includegraphics[width=0.15\textwidth]{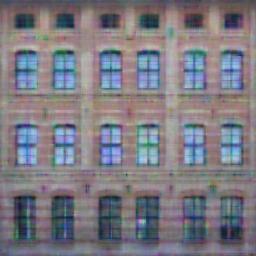} &
\includegraphics[width=0.15\textwidth]{image/Color_TC_TWCTV.jpg} \\
TNN~\cite{wenjin2022tnn} & TNN+TV~\cite{qiu2021tnntv} & IRTNN~\cite{wang2022irtnn} & SPC+TV~\cite{yokota2016spctv} & TCTV~\cite{Wang202310990} & TWCTV \\
\end{tabular}

\caption{Visual comparison of different tensor completion methods with an observed tensor having 95\% missing entries.}
\label{fig:comparison}
\end{figure}

\begin{table}[!htbp]
\caption{Results for color image inpainting under different sampling rates (SR).}
\vspace{2mm}
\centering
\setlength{\tabcolsep}{2.5mm}

\begin{tabular}{c|cccc|cccc}
\hline
\multirow{2}{*}{Image} 
& \multicolumn{4}{c|}{SR = 5\%} 
& \multicolumn{4}{c}{SR = 10\%} \\
\cline{2-9}
& PSNR & SSIM & FSIM & Time 
& PSNR & SSIM & FSIM & Time \\
\hline
Observed & 5.966 & 0.013 & 0.490 & 0.000 & 6.199 & 0.025 & 0.559 & 0.000 \\
SNN & 18.809 & 0.344 & 0.547 & 2.860 & 21.594 & 0.548 & 0.701 & 2.888 \\
SNN+TV & 22.265 & 0.684 & 0.833 & 11.686 & 24.871 & 0.793 & 0.884 & 12.432 \\
BCPF & 23.295 & 0.707 & 0.850 & 19.196 & 25.570 & 0.796 & 0.903 & 33.455 \\
KBR & 20.317 & 0.580 & 0.797 & 19.188 & 25.489 & 0.797 & 0.900 & 22.017 \\
TNN & 18.825 & 0.323 & 0.530 & 4.474 & 20.654 & 0.487 & 0.660 & 4.371 \\
TNN+TV & 21.636 & 0.668 & 0.828 & 14.924 & 24.972 & 0.802 & 0.895 & 15.716 \\
IRTNN & 20.771 & 0.616 & 0.821 & 119.125 & 24.557 & 0.779 & 0.897 & 55.370 \\
SPC+TV & 21.540 & 0.652 & 0.816 & 16.439 & 24.562 & 0.783 & 0.881 & 17.248 \\
TCTV & 22.586 & 0.656 & 0.796 & 20.710 & 25.994 & 0.820 & 0.899 & 21.333 \\
TWCTV & \textbf{25.337} & \textbf{0.799} & \textbf{0.893} & 17.097 & \textbf{27.884} & \textbf{0.872} & \textbf{0.934} & 17.314 \\
\hline
\end{tabular}

\vspace{5mm}

\begin{tabular}{c|cccc|cccc}
\hline
\multirow{2}{*}{Image} 
& \multicolumn{4}{c|}{SR = 20\%} 
& \multicolumn{4}{c}{SR = 50\%} \\
\cline{2-9}
& PSNR & SSIM & FSIM & Time 
& PSNR & SSIM & FSIM & Time \\
\hline
Observed & 6.709 & 0.055 & 0.598 & 0.000 & 8.751 & 0.178 & 0.636 & 0.000 \\
SNN & 25.314 & 0.777 & 0.859 & 2.379 & 31.426 & 0.949 & 0.972 & 1.910 \\
SNN+TV & 27.620 & 0.879 & 0.930 & 11.307 & 33.585 & 0.969 & 0.980 & 11.280 \\
BCPF & 28.501 & 0.886 & 0.940 & 32.093 & 33.277 & 0.960 & 0.975 & 36.049 \\
KBR & 28.861 & 0.894 & 0.946 & 20.620 & 35.393 & 0.977 & 0.987 & 22.974 \\
TNN & 23.733 & 0.719 & 0.819 & 3.879 & 29.857 & 0.935 & 0.964 & 3.886 \\
TNN+TV & 28.141 & 0.892 & 0.939 & 14.943 & 34.726 & 0.976 & 0.984 & 15.103 \\
IRTNN & 28.229 & 0.888 & 0.942 & 29.121 & 34.879 & 0.975 & 0.985 & 17.373 \\
SPC+TV & 27.599 & 0.878 & 0.929 & 11.166 & 32.181 & 0.960 & 0.974 & 8.302 \\
TCTV & 29.510 & 0.914 & 0.953 & 20.003 & 36.519 & 0.983 & 0.990 & 21.073 \\
TWCTV & \textbf{30.860} & \textbf{0.933} & \textbf{0.962} & 17.411 & \textbf{36.650} & \textbf{0.983} & \textbf{0.990} & 18.804 \\
\hline
\end{tabular}

\label{tab:imp1}
\end{table}

\begin{figure}[!ht]
\centering

\begin{tabular}{ccccc}
\includegraphics[width=0.15\textwidth]{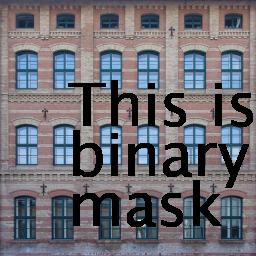} &
\includegraphics[width=0.15\textwidth]{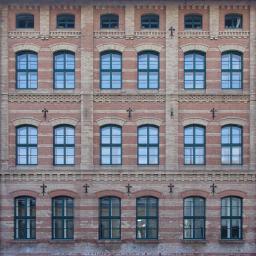} &
\includegraphics[width=0.15\textwidth]{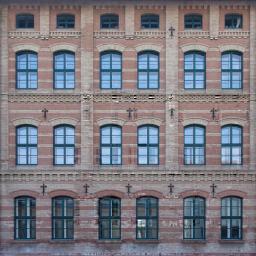} &
\includegraphics[width=0.15\textwidth]{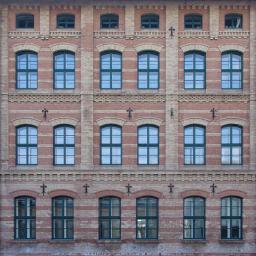} &
\includegraphics[width=0.15\textwidth]{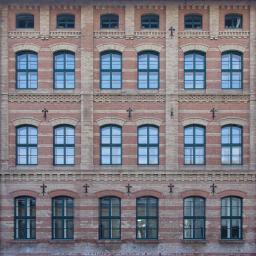} \\
Observed & SNN~\cite{liu2012snn} & BCPF~\cite{zhao2015bcpf} & KBR~\cite{Qi2018kbr} & TNN~\cite{wenjin2022tnn} \\

\includegraphics[width=0.15\textwidth]{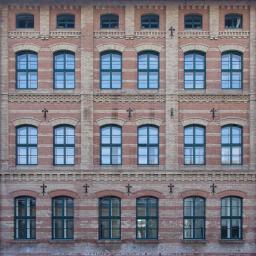} &
\includegraphics[width=0.15\textwidth]{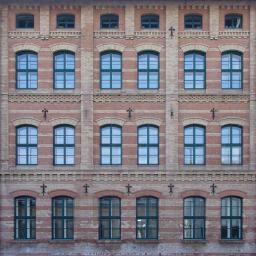} &
\includegraphics[width=0.15\textwidth]{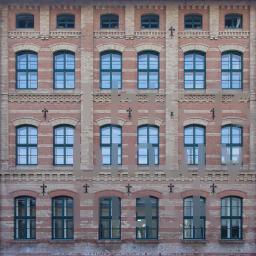} &
\includegraphics[width=0.15\textwidth]{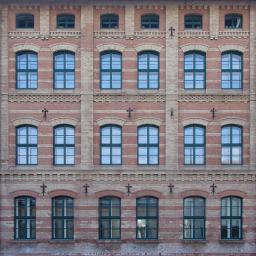} &
\includegraphics[width=0.15\textwidth]{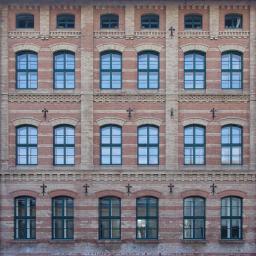} \\
IRTNN~\cite{wang2022irtnn} & SPC+TV~\cite{yokota2016spctv} & TNN+TV~\cite{qiu2021tnntv} & TCTV~\cite{Wang202310990} & TWCTV \\
\end{tabular}

\caption{Visual comparison of different tensor completion methods under a text mask. The first row shows the observed image and results from SNN, BCPF, KBR, and TNN. The second row shows results from IRTNN, SPC+TV, TNN+TV, TCTV, and the proposed TWCTV method.}
\label{fig:comparison_tc_text}
\end{figure}

\begin{table}[!ht]
\caption{Quantitative comparison of different methods on color image inpainting using text mask.}
\vspace{2mm}
\centering

\setlength{\tabcolsep}{2.5mm}{
   \begin{tabular}{cccccc}
       \hline
       \multicolumn{1}{c}{} &\multicolumn{1}{c}{Image} & \multicolumn{1}{c}{PSNR} & \multicolumn{1}{c}{SSIM} & \multicolumn{1}{c}{FSIM} & \multicolumn{1}{c}{Time} \\
       \hline
       & \multicolumn{1}{c}{Observed} &\multicolumn{1}{c}{14.957} & \multicolumn{1}{c}{0.786} &\multicolumn{1}{c}{0.772} & \multicolumn{1}{c}{0.000} \\
       & \multicolumn{1}{c}{SNN} &\multicolumn{1}{c}{33.706} & \multicolumn{1}{c}{0.977} &\multicolumn{1}{c}{0.983} & \multicolumn{1}{c}{1.801} \\
       & \multicolumn{1}{c}{BCPF} &\multicolumn{1}{c}{29.883} & \multicolumn{1}{c}{0.948} &\multicolumn{1}{c}{0.967} & \multicolumn{1}{c}{55.891} \\
       & \multicolumn{1}{c}{KBR} &\multicolumn{1}{c}{34.547} & \multicolumn{1}{c}{0.978} &\multicolumn{1}{c}{0.986} & \multicolumn{1}{c}{20.953} \\
       & \multicolumn{1}{c}{TNN} &\multicolumn{1}{c}{28.966} & \multicolumn{1}{c}{0.928} &\multicolumn{1}{c}{0.945} & \multicolumn{1}{c}{4.180} \\
      
       & \multicolumn{1}{c}{IRTNN} &\multicolumn{1}{c}{33.242} & \multicolumn{1}{c}{0.973} &\multicolumn{1}{c}{0.982} & \multicolumn{1}{c}{55.065} \\
       & \multicolumn{1}{c}{SPC+TV} &\multicolumn{1}{c}{33.546} & \multicolumn{1}{c}{0.975} &\multicolumn{1}{c}{0.983} & \multicolumn{1}{c}{4.509} \\
        & \multicolumn{1}{c}{TNN+TV} &\multicolumn{1}{c}{33.639} & \multicolumn{1}{c}{0.976} &\multicolumn{1}{c}{0.983} & \multicolumn{1}{c}{15.136} \\
       & \multicolumn{1}{c}{TCTV} &\multicolumn{1}{c}{{35.139}} & \multicolumn{1}{c}{{0.979}} &\multicolumn{1}{c}{{0.986}} & \multicolumn{1}{c}{22.218} \\
       & \multicolumn{1}{c}{TWCTV} &\multicolumn{1}{c}{\bf 35.452} & \multicolumn{1}{c}{\bf 0.980} &\multicolumn{1}{c}{\bf 0.987} & \multicolumn{1}{c}{17.688} \\
       \hline
   \end{tabular}
   
}
\label{tab:quantitative_results}

\end{table}

\subsubsection{MSI impainting}

In this experiment, the proposed TWCTV algorithm is applied to the task of multispectral image (MSI) inpainting and compared with several state-of-the-art methods. We evaluate on the ``cloth\_ms'' data\footnote{\url{https://cave.cs.columbia.edu/repository/Multispectral/Stuff}} of size $256\times256\times31$, with 5\% randomly sampled entries. The performance of the recovery methods is evaluated using both quantitative metrics and visual inspection. Table~\ref{table:MSI_inpainting_results} presents the quantitative results, while Figure~\ref{fig:pseudo_color_eight} provides a pseudo-color visualization of the reconstruction results using three representative bands (R:25, G:15, B:5).

The better performance of the proposed TWCTV method is mainly due to its use of an exponential weighting mechanism, which helps to improve the recovery results. By assigning higher weights to more significant singular values and lower weights to less relevant ones, TWCTV ensures that key spatial and spectral details are preserved during the reconstruction process. This selective emphasis on important components allows TWCTV to better handle the high sparsity in the data, especially when a large portion of the pixels (95\%) is missing.

\begin{figure}[!ht]
   \centering
   \begin{tabular}{@{}c@{\hspace{1mm}}c@{\hspace{1mm}}c@{\hspace{1mm}}c@{\hspace{1mm}}c@{\hspace{1mm}}c@{\hspace{1mm}}c@{\hspace{1mm}}c@{}}
       \includegraphics[width=0.12\textwidth]{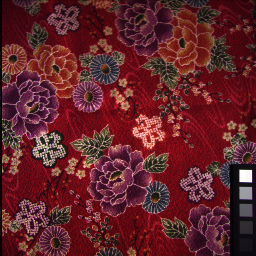} &
       \includegraphics[width=0.12\textwidth]{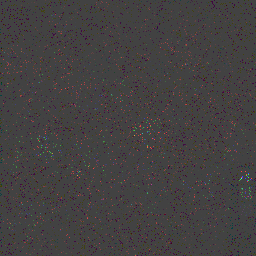} &
       \includegraphics[width=0.12\textwidth]{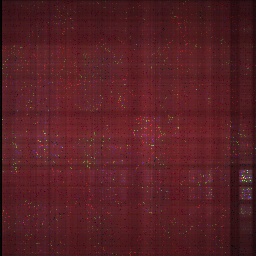} &
       \includegraphics[width=0.12\textwidth]{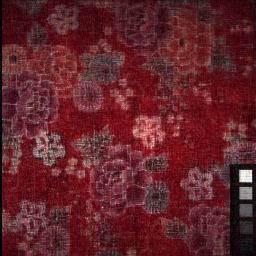} &
       \includegraphics[width=0.12\textwidth]{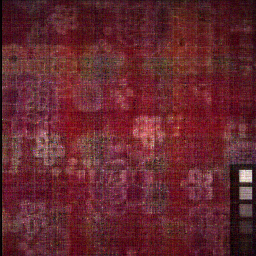} &
       \includegraphics[width=0.12\textwidth]{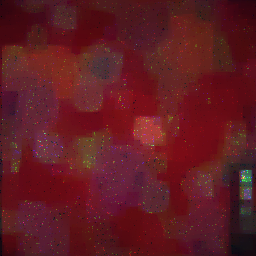} &
       \includegraphics[width=0.12\textwidth]{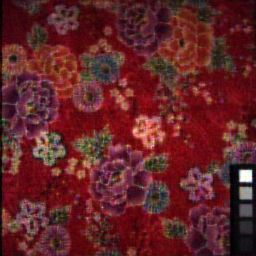} &
       \includegraphics[width=0.12\textwidth]{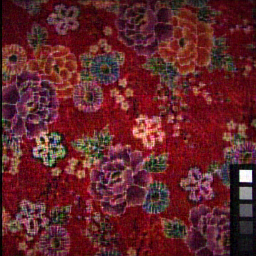} \\
      Original & Observed & SNN~\cite{liu2012snn} & KBR~\cite{Qi2018kbr} & TNN~\cite{wenjin2022tnn} & TNN+TV~\cite{qiu2021tnntv} & TCTV\cite{Wang202310990} & TWCTV \\
       
   \end{tabular}
   \caption{Pseudo-color visualization of  MSI completion methods using bands (R:25, G:15, B:5).}
   \label{fig:pseudo_color_eight}
\end{figure}

\begin{table}[!ht]
\caption{Quantitative results for MSI image inpainting with 5\% sampling rate.}
\vspace{2mm}
\centering

\setlength{\tabcolsep}{2.5mm}{
   \begin{tabular}{ccccc}
       \hline
       \multicolumn{1}{c}{} &\multicolumn{1}{c}{Image} & \multicolumn{1}{c}{PSNR} & \multicolumn{1}{c}{SSIM} & \multicolumn{1}{c}{ERGAS} \\
       \hline
       & \multicolumn{1}{c}{Observed} &\multicolumn{1}{c}{17.956} & \multicolumn{1}{c}{0.230} & \multicolumn{1}{c}{544.529} \\
       & \multicolumn{1}{c}{SNN} &\multicolumn{1}{c}{20.665} & \multicolumn{1}{c}{0.331} & \multicolumn{1}{c}{395.775} \\
       & \multicolumn{1}{c}{KBR} &\multicolumn{1}{c}{25.133} & \multicolumn{1}{c}{0.624} & \multicolumn{1}{c}{232.986} \\
       
       & \multicolumn{1}{c}{TNN} &\multicolumn{1}{c}{21.608} & \multicolumn{1}{c}{0.388} & \multicolumn{1}{c}{361.431} \\
       & \multicolumn{1}{c}{TNN+TV} &\multicolumn{1}{c}{22.249} & \multicolumn{1}{c}{0.423} & \multicolumn{1}{c}{321.439} \\
       & \multicolumn{1}{c}{TCTV} &\multicolumn{1}{c}{25.096} & \multicolumn{1}{c}{0.661} & \multicolumn{1}{c}{231.123} \\
       & \multicolumn{1}{c}{TWCTV} &\multicolumn{1}{c}{\textbf{25.864}} & \multicolumn{1}{c}{\textbf{0.699}} & \multicolumn{1}{c}{\textbf{211.465}} \\
       \hline
   \end{tabular}
}
\label{table:MSI_inpainting_results}
\end{table}

\begin{figure}[!ht]
\centering
\begin{tabular}{cc}
\includegraphics[width=0.5\textwidth]{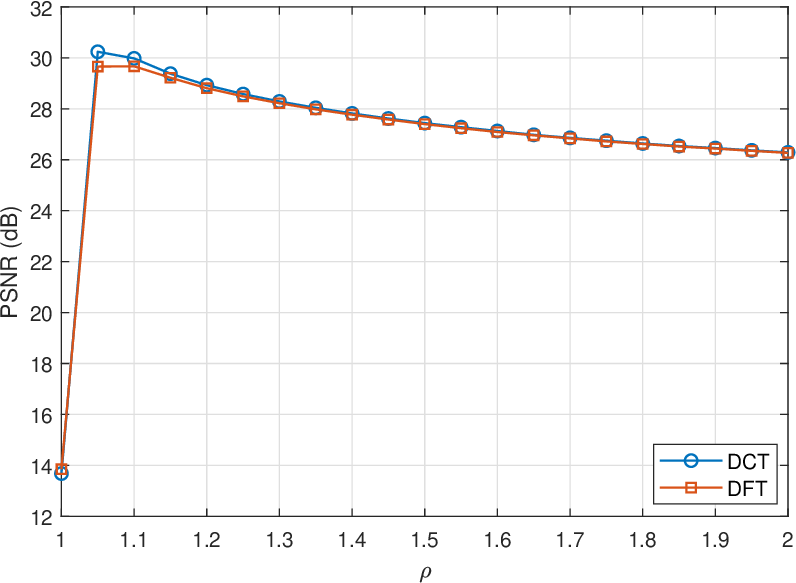} &
\includegraphics[width=0.5\textwidth]{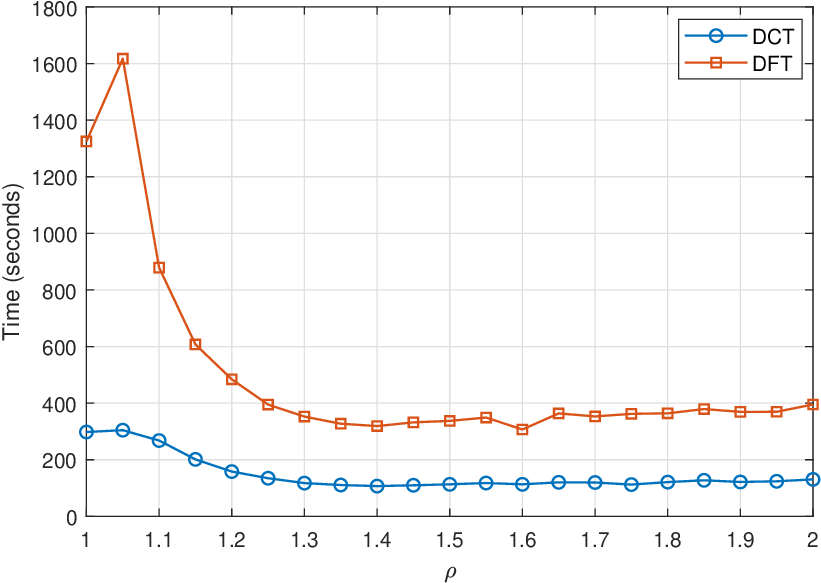} \\
(a) PSNR vs $\rho$ & (b) Time vs $\rho$
\end{tabular}
\caption{Parameter sensitivity for the penalty parameter $\rho$ for TWCTV based tensor completion on MSI image (size: $256 \times 256 \times 31$) with 80\% missing entries.}
\label{fig:parameter_sensitivity_tc}
\end{figure}
The sensitivity analysis of tuning penalty parameter $\rho$ in Algorithm~\ref{alg:TWCTVTRPCA} for TWCTV-based tensor completion is presented in Figure~\ref{fig:parameter_sensitivity_tc}, using MSI data with 80\% missing entries. The PSNR curves show that using $\rho$ values between 1.1 and 1.3 gives stable and high reconstruction quality (about 30\,dB for the DCT matrix and 29.6\,dB for the DFT matrix), and helps avoid the sharp drop in performance at higher $\rho$ values. In terms of computational efficiency, DCT consistently requires only about one-third of the processing time compared to DFT (e.g., $\sim200$s vs $\sim600$s at $\rho=1.15$), making it significantly more practical for real-world applications.

\subsection{ Image denoising using tensor RPCA}
In this subsection, we present the results of denoising experiments conducted on color image and HSI using the tensor RPCA-based method, which effectively separates sparse noise from the low-rank tensor structure inherent in multi-channel image data. A thorough comparison is made between the proposed method and several existing state-of-the-art approaches. The first group of methods includes SNN \cite{liu2012snn}, KBR \cite{Qi2018kbr}, and TNN \cite{wenjin2022tnn}, while the second group consists of LRTV \cite{he2016lrtv}, LRTDTV \cite{wang2018lrtdtv}, TLR-HTV \cite{chen2018tlrhtv}, TCTV \cite{Wang202310990}, and the proposed TWCTV-based TRPCA method. For our experiments, we set $p=0.9$, $M=\text{DCT matrix}$ for the proposed method and $\lambda = \frac{1}{\sqrt{\frac{n_1 n_2 n_3}{\min(n_1, n_2)}}}$ across all compared methods. We assess recovery performance through multiple quality metrics: PSNR captures pixel-wise reconstruction accuracy, SSIM evaluates structural and perceptual similarity, FSIM measures the preservation of low-level image features, and ERGAS provides a global measure of spectral distortion particularly relevant for hyperspectral data.

\subsubsection{Color image denoising}
We evaluate TWCTV-based TRPCA denoising performance against existing methods on the ``starfish'' image \footnote{\url{https://www2.eecs.berkeley.edu/Research/Projects/CS/vision/bsds/}} with a size of $256\times256\times3$ corrupted with salt-and-pepper noise at different levels (10\%, 30\%, 50\%).  As shown in Table~\ref{tab:trpcacoldenoise}, TWCTV consistently outperforms existing methods, including SNN, KBR, TNN, LRTV, and TCTV across all noise levels (0.1, 0.3, 0.5), achieving superior results in all metrics while maintaining competitive computational efficiency.

\begin{figure}[!ht]
\centering

\begin{tabular}{ccccc}
\includegraphics[width=0.18\textwidth]{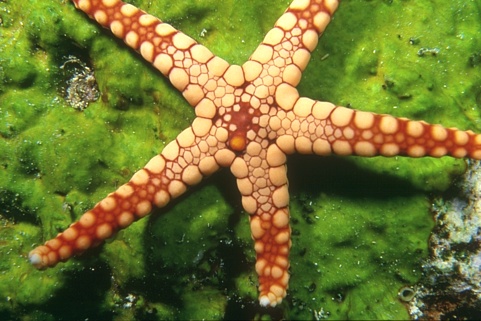} &
\includegraphics[width=0.18\textwidth]{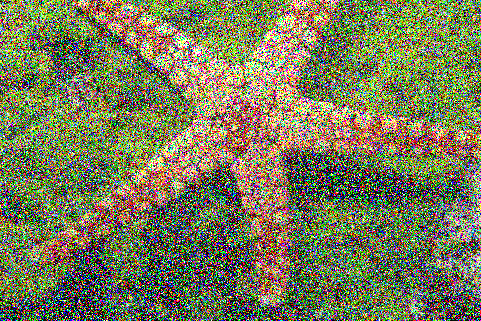} &
\includegraphics[width=0.18\textwidth]{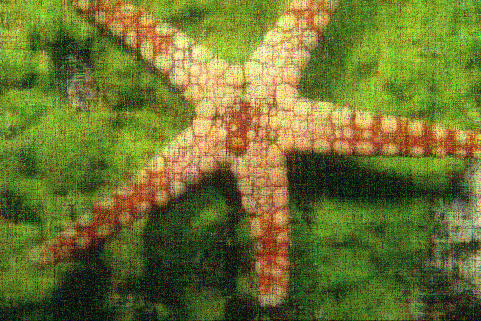} &
\includegraphics[width=0.18\textwidth]{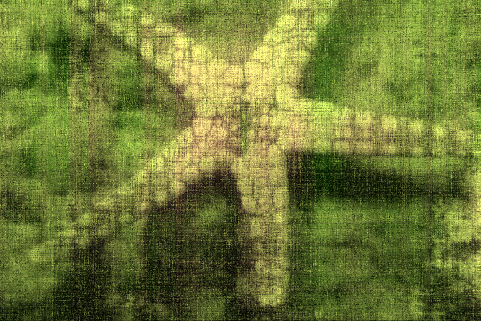} &
\includegraphics[width=0.18\textwidth]{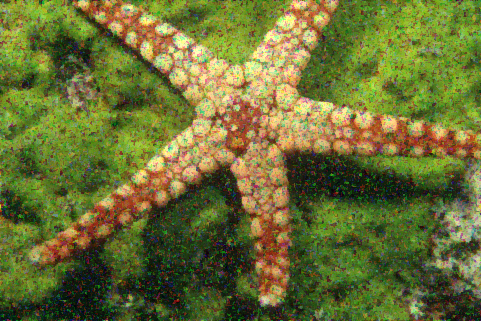} \\
Original & Observed & SNN~\cite{liu2012snn} & KBR~\cite{Qi2018kbr} & LRTV~\cite{he2016lrtv} \\

\includegraphics[width=0.18\textwidth]{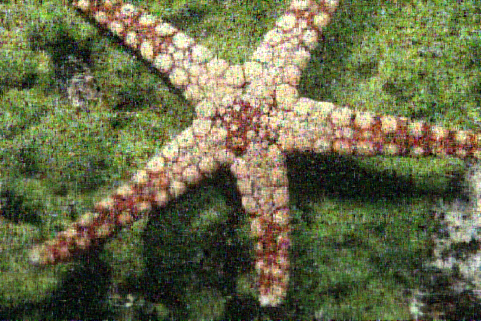} &
\includegraphics[width=0.18\textwidth]{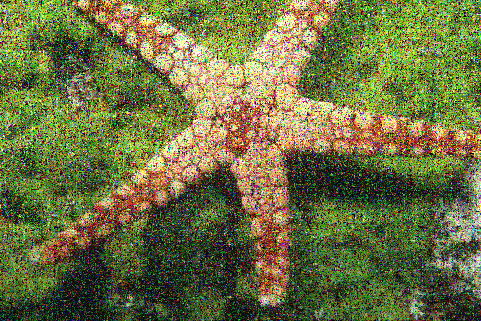} &
\includegraphics[width=0.18\textwidth]{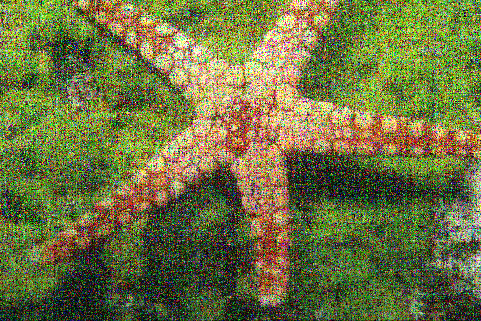} &
\includegraphics[width=0.18\textwidth]{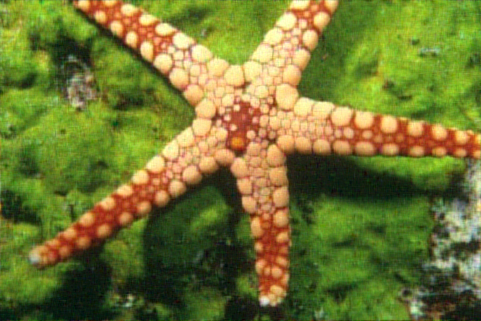} &
\includegraphics[width=0.18\textwidth]{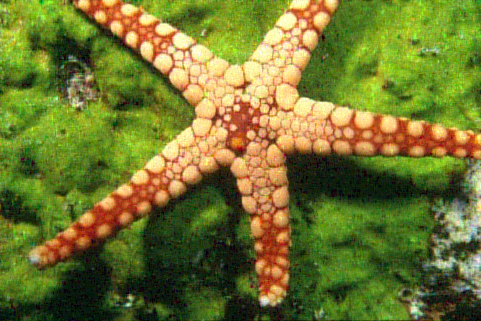} \\
LRTDTV~\cite{wang2018lrtdtv} & TLR-HTV~\cite{chen2018tlrhtv} & TNN~\cite{wenjin2022tnn} & TCTV~\cite{Wang202310990} & TWCTV \\
\end{tabular}

\caption{Visual comparison of different tensor RPCA methods on ``starfish''. The first row shows the original image, observed noisy tensor (Noise level = 0.5), and results from SNN, KBR, and LRTV. The second row shows results from LRTDTV, TLR-HTV, TNN, TCTV, and the proposed TWCTV method.}
\label{fig:comparison_rpca}
\end{figure}

\begin{table}[!ht]
\caption{Comparison of different methods on color image denoising under different levels of salt-and-pepper noise.}
\centering
\setlength{\tabcolsep}{2.5mm}
\begin{tabular}{c|ccc|ccc|ccc}
\hline
\multirow{2}{*}{Method} 
& \multicolumn{3}{c|}{Noise = 0.5} 
& \multicolumn{3}{c|}{Noise = 0.3} 
& \multicolumn{3}{c}{Noise = 0.1} \\
\cline{2-10}
& PSNR & SSIM & FSIM 
& PSNR & SSIM & FSIM 
& PSNR & SSIM & FSIM \\
\hline
Noisy & 7.803 & 0.084 & 0.454 & 10.012 & 0.155 & 0.535 & 14.803 & 0.369 & 0.738 \\
SNN & 18.341 & 0.397 & 0.742 & 24.552 & 0.796 & 0.887 & 27.996 & 0.917 & 0.944 \\
KBR & 17.890 & 0.324 & 0.700 & 21.271 & 0.678 & 0.845 & 23.929 & 0.883 & 0.931 \\
LRTV & 19.328 & 0.506 & 0.774 & 24.085 & 0.748 & 0.884 & 29.381 & 0.916 & 0.955 \\
LRTDTV & 18.012 & 0.459 & 0.745 & 22.602 & 0.715 & 0.868 & 29.209 & 0.899 & 0.947 \\
TLR-HTV & 12.591 & 0.191 & 0.577 & 22.194 & 0.642 & 0.853 & 32.478 & 0.948 & 0.973 \\
TNN & 12.313 & 0.167 & 0.574 & 23.831 & 0.713 & 0.877 & 31.722 & 0.962 & 0.975 \\
TCTV & 23.385 & 0.725 & 0.807 & 29.138 & 0.909 & 0.949 & 31.482 & 0.950 & 0.973 \\
TWCTV & \textbf{26.814} & \textbf{0.810} & \textbf{0.896} & \textbf{31.414} & \textbf{0.942} & \textbf{0.972} & \textbf{33.394} & \textbf{0.964} & \textbf{0.982} \\
\hline
\end{tabular}
\label{tab:trpcacoldenoise}
\end{table}

\begin{figure}[!ht]
\centering
\begin{tabular}{cc}
\includegraphics[width=0.5\textwidth]{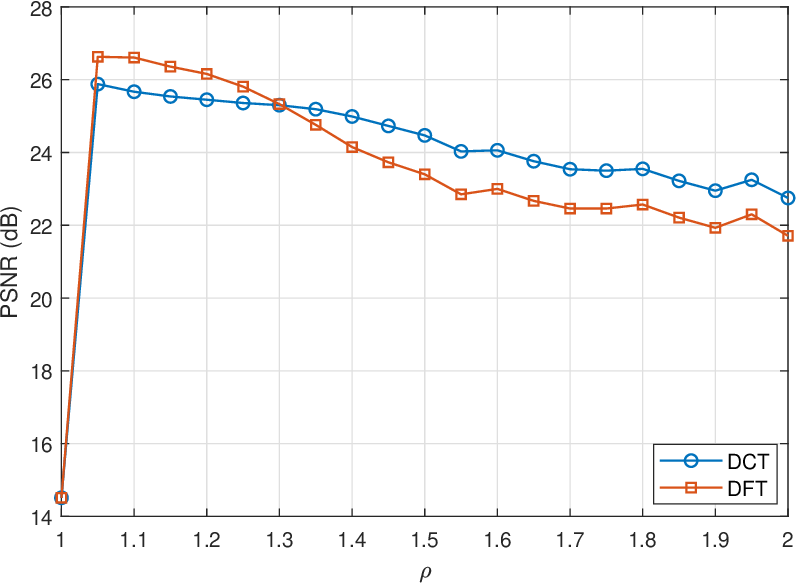} &
\includegraphics[width=0.5\textwidth]{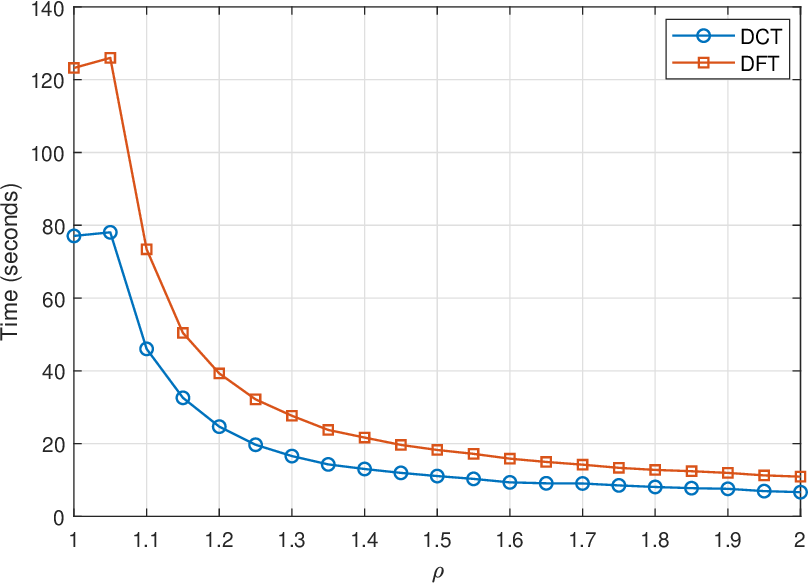} \\
(a) PSNR vs $\rho$. & (b) Time vs $\rho$.
\end{tabular}
\caption{Parameter sensitivity for the parameter $\rho$ for TWCTV based TRPCA on color image (size: $256 \times 256 \times 3$) with 50\% impulse noise or salt-paper noise.}
\label{fig:parameter_sensitivity_trpca}
\end{figure}
Figure~\ref{fig:parameter_sensitivity_trpca} demonstrates the effect of penalty parameter $\rho$ on color image denoising via TWCTV-based tensor RPCA Algorithm~\ref{alg:TWCTVTRPCA}.
Analysis of PSNR values reveals that $\rho$ values between $1.05$ to $1.15$ achieve optimal denoising performance (around $26$ dB for DFT matrix, $25.5$ dB for DCT matrix), after which the performance gradually declines. The computational comparison shows DCT maintains consistently faster processing times, requiring approximately 40\% less computation time than DFT (e.g., $\sim32$s vs $\sim50$s at $\rho=1.15$).

\subsubsection{HSI denoising}

To evaluate the robustness of the proposed methods under varying noise conditions, we used the ``hsi-Pac'' dataset \footnote{\url{https://www.ehu.eus/ccwintco/index.php/Hyperspectral_Remote_Sensing_Scenes}}
 with a size of \(200 \times 200 \times 80\) as input. Sparse salt-and-pepper noise was introduced at three levels: 10\%, 30\%, and 50\%, creating increasingly challenging conditions for denoising. The quantitative results for all methods under these sparse error rates (SER) are summarized in Table~\ref{tab:hsi_denoising}. Also visual comparison of different methods for HSI denoising with sparse rate 30 \% is shown in Figure~\ref{fig:comparisonhsi}.

\begin{figure}[!h]
\centering

\begin{tabular}{ccccc}
\includegraphics[width=0.18\textwidth]{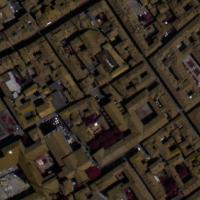} &
\includegraphics[width=0.18\textwidth]{image/HSI_Noisy.jpg} &
\includegraphics[width=0.18\textwidth]{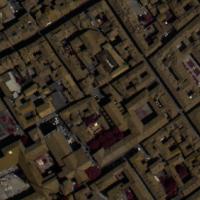} &
\includegraphics[width=0.18\textwidth]{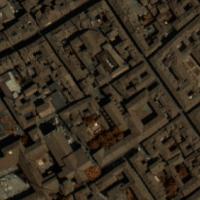} &
\includegraphics[width=0.18\textwidth]{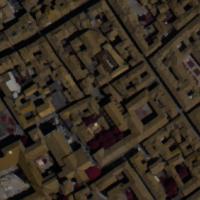} \\
Original & Noisy & SNN~\cite{liu2012snn} & KBR~\cite{Qi2018kbr} & LRTV~\cite{he2016lrtv} \\

\includegraphics[width=0.18\textwidth]{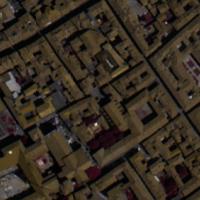} &
\includegraphics[width=0.18\textwidth]{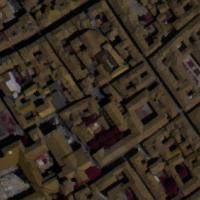} &
\includegraphics[width=0.18\textwidth]{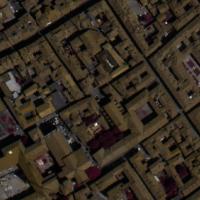} &
\includegraphics[width=0.18\textwidth]{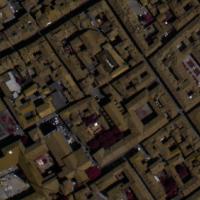} &
\includegraphics[width=0.18\textwidth]{image/HSI_TWCTV.jpg} \\
LRTDTV~\cite{wang2018lrtdtv} & TLR-HTV~\cite{chen2018tlrhtv} & TNN~\cite{wenjin2022tnn} & TCTV~\cite{Wang202310990} & TWCTV \\
\end{tabular}

\caption{Visual comparison of different methods for HSI denoising with a sparse noise rate of 30\%.}
\label{fig:comparisonhsi}
\end{figure}

In comparison, methods such as SNN, TNN and TCTV also perform well under lower noise levels but struggle as the noise intensity increases, leaving residual errors or excessively smoothing the data, thereby losing important details. Other methods, including KBR and LRTV, are unable to achieve an effective trade-off between noise suppression and detail preservation, especially under higher SERs. The sparse weighting strategy in TWCTV, combined with its capability to explore the low-rank characteristics of tensors, ensures effective noise suppression and superior recovery of clean, high-quality data. This experiment confirms that TWCTV is particularly well-suited for handling sparse noise due to its adaptive weighting mechanism, which prioritizes essential features and minimizes the impact of noise, resulting in accurate denoising and better preservation of local structures and edge information in HSI data.

{\small{
\begin{table}[!ht]
\caption{Performance comparison of various methods for HSI denoising under different sparse error rates (SER).}
\centering
\setlength{\tabcolsep}{1.8mm}{
\begin{tabular}{c|ccc|ccc|ccc}
\hline
\multirow{2}{*}{Method} 
& \multicolumn{3}{c|}{SER = 10\%} 
& \multicolumn{3}{c|}{SER = 30\%} 
& \multicolumn{3}{c}{SER = 50\%} \\
\cline{2-10}
& PSNR & ERGAS & SSIM
& PSNR & ERGAS & SSIM
& PSNR & ERGAS & SSIM \\
\hline
Noisy   & 14.524 & 699.002 & 0.275 & 9.752 & 1210.860 & 0.075 & 7.526 & 1564.154 & 0.032 \\
SNN     & 43.328 & 35.941 & 0.997 & 38.694 & 58.770 & 0.990 & 33.838 & 94.077 & 0.971 \\
KBR     & 34.127 & 77.090 & 0.970 & 31.593 & 100.917 & 0.949 & 24.538 & 216.427 & 0.709 \\
LRTV    & 38.536 & 44.482 & 0.985 & 34.863 & 68.429 & 0.964 & 30.786 & 109.954 & 0.909 \\
LRTDTV  & 38.181 & 46.094 & 0.979 & 36.865 & 53.601 & 0.974 & 33.734 & 77.380 & 0.950 \\
TLR-HTV & 36.383 & 61.872 & 0.974 & 33.163 & 84.792 & 0.948 & 29.606 & 124.096 & 0.883 \\
TNN     & 46.864 & 44.200 & 0.990 & 41.780 & 55.692 & 0.983 & 26.169 & 182.379 & 0.724 \\
TCTV    & 48.799 & 25.133 & 0.996 & 44.933 & 33.191 & 0.992 & 39.888 & 45.287 & 0.977 \\
TWCTV   & \textbf{51.959} & \textbf{20.328} & \textbf{0.998} 
        & \textbf{45.621} & \textbf{27.487} & \textbf{0.995} 
        & \textbf{41.092} & \textbf{42.918} & \textbf{0.986} \\
\hline
\end{tabular}
\label{tab:hsi_denoising}
}
\end{table}
}}

\subsection{Salient object detection via background removal using TRPCA}

In our exploration of salient object detection, we aim to distinguish the foreground object (salient object) from the background. This challenge is tackled by using the low-rank tensor property of background data and the sparse tensor representation of foreground objects using the TRPCA framework.

To evaluate our approach, we employ the ChangeDetection.net (CDNet) 2014 dataset \cite{wang2014cdnet}, which includes diverse video sequences. Here we focus on the ``highway'' video sequence (frames 900-1000), represented as a tensor of size $h\times w\times 101 \times 3$, with $h \times w$ denoting the frame size, 101 indicating the number of frames, and three corresponding to the RGB channels. We utilize various methods for comparison, including matrix-based RPCA~\cite{candes2011robust}, Tucker decomposition-based SNN~\cite{huang2015provable}, t-SVD based TNN \cite{wenjin2022tnn}, and tensor correlated total variation (TCTV)~\cite{Wang202310990} with the proposed method TWCTV based tensor RPCA. The data is processed differently for matrix and tensor methods: RPCA converts frames to a $3n_1 \times 101$ matrix ($n_1 = h\times w$), while tensor-based approaches (SNN, TNN, TCTV, TWCTV) use a $n_1 \times 101 \times 3$ format.

The regularization parameter $\lambda$ is set to $\frac{1}{\max \{3 n_1, n_2\}}$ for RPCA and $\frac{1}{3 \max \{n_1, n_2\}}$ for SNN,TNN,TCTV and TWCTV. All other parameters for these methods follow the respective specifications in the original papers.

\begin{figure}
  \centering
 \begin{tabular}{@{}c@{\hspace{2mm}}c@{\hspace{2mm}}c@{\hspace{2mm}}c@{\hspace{2mm}}c@{\hspace{2mm}}c@{\hspace{2mm}}c@{}}
    \includegraphics[width=0.13\textwidth]{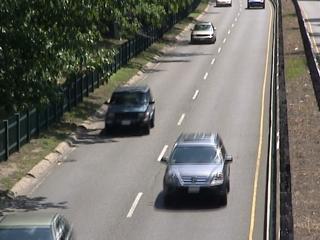} &
    \includegraphics[width=0.13\textwidth]{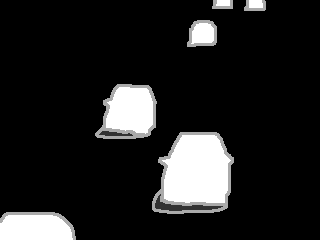} &
    \includegraphics[width=0.13\textwidth]{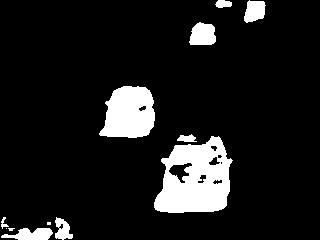} &
    \includegraphics[width=0.13\textwidth]{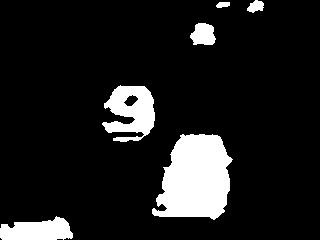} &
    \includegraphics[width=0.13\textwidth]{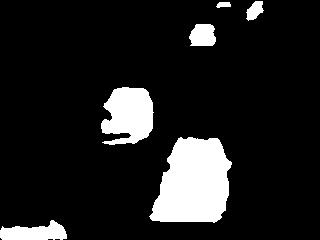} &
    \includegraphics[width=0.13\textwidth]{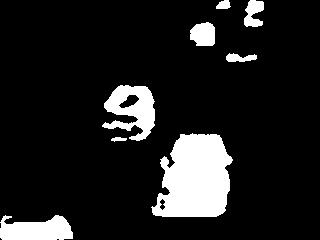} &
    \includegraphics[width=0.13\textwidth]{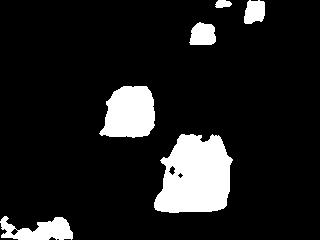} \\
    \includegraphics[width=0.13\textwidth]{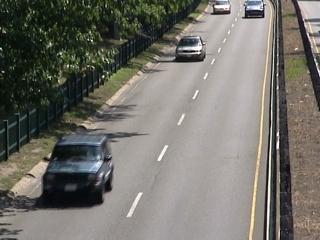} &
    \includegraphics[width=0.13\textwidth]{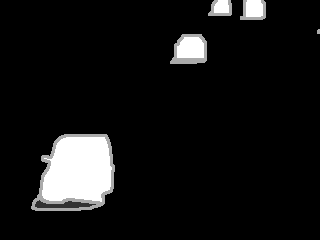} &
    \includegraphics[width=0.13\textwidth]{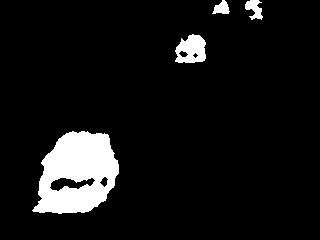} &
    \includegraphics[width=0.13\textwidth]{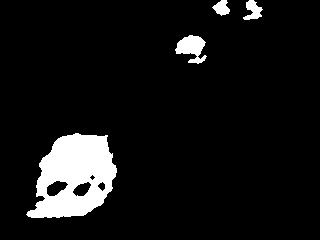} &
    \includegraphics[width=0.13\textwidth]{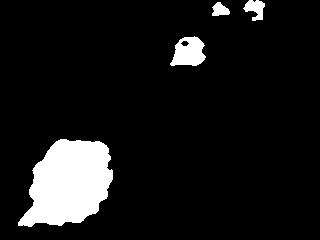} &
    \includegraphics[width=0.13\textwidth]{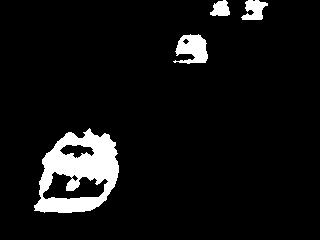} &
    \includegraphics[width=0.13\textwidth]{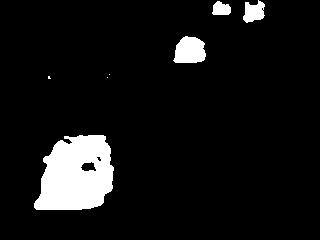} \\
  \makebox[0.12\textwidth][c]{Input}&\makebox[0.12\textwidth][c]{Ground Truth}&\makebox[0.12\textwidth][c]{RPCA}~\cite{candes2011robust}&
  \makebox[0.12\textwidth][c]{SNN}\cite{liu2012snn}&\makebox[0.12\textwidth][c]{TNN}\cite{wenjin2022tnn}&\makebox[0.12\textwidth][c]{TCTV}\cite{Wang202310990}&
  \makebox[0.12\textwidth][c]{TWCTV}
\end{tabular}
  \caption{Foreground detection results on frames 900 (top) and 925 (bottom) comparing different methods.}
    \label{fig:foreground_detection}
\end{figure}

\begin{table}[h!]
\caption{Performance comparison of foreground detection on the CDNet dataset.}
\vspace{2mm}
\centering

\begin{tabular}{|c|c|c|c|c|c|c|}
\hline
Dataset & Metrics & RPCA & SNN & TNN & TCTV & TWCTV \\
\hline
\multirow{3}{*}{highway} 
& Precision & 0.8873  & 0.8966 & 0.8566 & 0.8232 & 0.8990 \\
& Recall & 0.7886 & 0.7092  & 0.8690 & 0.7992 & 0.8866 \\
& F-Measure &  0.8348 & 0.7923 & 0.8522 & 0.8110 & 0.8927 \\

\hline
\end{tabular}

\label{table:video_results}
\end{table}

The performance of foreground detection is evaluated using Recall, Precision, and their harmonic mean, the F-measure. Recall and Precision assess foreground detection by measuring correctly identified pixels (TP) relative to total actual foreground pixels (TP + FN) and detected foreground pixels (TP + FP), respectively. The F-Measure is then calculated as:  
\[
\text{F-Measure} = 2 \cdot \frac{\text{Precision} \cdot \text{Recall}}{\text{Precision} + \text{Recall}}.
\]
 To generate binary masks for background frames, we apply hard thresholding to the sparse tensor slices $\mathcal{E}$ in~\eqref{eq:RTCmodel}, using the standard deviation of the corresponding slice of $\mathcal{E}$ as the threshold value. A median filtering operation is performed on each frame with a $5 \times 5$ window for further noise reduction. 

The proposed TWCTV-based TRPCA method enhances performance by incorporating weighted Schatten-$p$ norm on gradient tensor for low-rank regularization and adaptive weighting for the sparse component, outperforming traditional non-weighted and convex methods. Table~\ref{table:video_results} shows that our TWCTV method achieves superior salient object detection performance on the CDNet datasets in terms of the F-Measure metric. Moreover, the qualitative results in Figure~\ref{fig:foreground_detection} demonstrate our method's improved ability to capture ground truth, validating the nonconvex regularizer as a surrogate for tubal rank.

\section{Conclusion}

This paper proposed a single non-convex regularizer (TWCTV) with a weighted Schatten-$p$ norm on gradient tensors, enhancing the characterization of low-rankness and smoothness in tensors while removing the need for a balancing parameter between these terms. By incorporating exponential weight for low-rank regularization in the gradient domain and exponential functions for sparse weight modeling, the proposed RLRTC model adaptively preserves significant tensor information. Integrated into tensor completion and TRPCA tasks, the proposed method uses an ADMM-based optimization framework with proven convergence. The non-convexity of the low-rank regularizer and adaptive weighting mechanisms of sparse terms effectively reduce issues like ill-conditioning and bias found in the traditional convex approach while preserving key singular values and significant sparse entries. Future research directions include exploring new nonconvex regularizers on gradient tensors and deep unfolded network-based ADMM to enhance tensor recovery performance.

\vspace{-.4cm}
\section*{Funding}
 Ratikanta Behera is supported by the Anusandhan National Research Foundation (ANRF), Government of India, under Grant No. EEQ/2022/001065.
\vspace{-.5cm}
\section*{Conflict of Interest}
The authors would like to assure the readers that they have no potential conflicts of interest to report.
\section*{Data Availability}
In this study, no new data was generated or analyzed. All data sets referenced in this article are publicly available from the sources cited within the manuscript.

\vspace{-.4cm}
\section*{ORCID}
{\small{
Biswarup Karmakar~\orcidC \href{https://orcid.org/0009-0003-5635-5425}{ \hspace{2mm}\textcolor{lightblue}{https://orcid.org/0009-0003-5635-5425}} \\
Ratikanta Behera~\orcidA \href{https://orcid.org/0000-0002-6237-5700}{ \hspace{2mm}\textcolor{lightblue}{ https://orcid.org/0000-0002-6237-5700}}\\
}}
\vspace{-1cm}



\bibliographystyle{abbrv}
\bibliography{References}
\end{document}